\documentclass{article}
\usepackage{ijcai18}

\usepackage{times}
\usepackage{xcolor}
\usepackage{soul}
\usepackage[utf8]{inputenc}
\usepackage[small]{caption}

\usepackage{algorithm,algorithmic}
\usepackage{amsmath}
\usepackage{amsthm}
\usepackage{url}

\usepackage{amssymb, multirow, paralist, color}
\newtheorem{thm}{Theorem}
\newtheorem{prop}{Proposition}
\newtheorem{lemma}{Lemma}

\newtheorem{rethm}{Theorem}
\newtheorem{reprop}{Proposition}
\newtheorem{relemma}{Lemma}

\newtheorem*{thm*}{Theorem}
\newtheorem*{prop3*}{Proposition 3}
\newtheorem*{lemma*}{Lemma}

\usepackage{graphicx} 
\usepackage{subfigure}

\usepackage{hyperref}
\hypersetup{
  colorlinks=true,
  linkcolor=black,
  citecolor=black,
  filecolor=black,
  urlcolor=black
}

\def \y {\mathbf{y}}
\def \E {\mathrm{E}}
\def \x {\mathbf{x}}

\def \z {\mathbf{z}}

\def \R {\mathbb{R}}
\def \S {\mathcal{S}}

\def \A {\mathcal{A}}
\def \q {\mathbf{q}}
\def \v {\mathbf{v}}

\def \p {\mathbf{p}}
\def \q {\mathbf{q}}

\def \xh {\widehat{\x}}

\def \G {\mathcal G}
\def\grad{\mathcal G}

\newcommand\mycomment[1]{
}
\def\data{\mathbf{q}}

\title{A Unified Analysis of Stochastic Momentum Methods for Deep Learning}

\author{} 

\author{
Yan Yan$^{1,2}$, 
Tianbao Yang$^3$, 
Zhe Li$^3$, 
Qihang Lin$^4$,
Yi Yang$^{1,2}$
\\ 
$^1$ SUSTech-UTS Joint Centre of CIS, Southern University of Science and Technology \\
$^2$ Centre for Artificial Intelligence, University of Technology Sydney \\
$^3$ Department of Computer Science, The University of Iowa \\
$^4$ Tippie College of Business, The University of Iowa  \\
yan.yan-3@student.uts.edu.au,
\{tianbao-yang, zhe-li-1, qihang-lin\}@uiowa.edu,
yi.yang@uts.edu.au
}
\begin{document}

\maketitle

\begin{abstract}
{\it Stochastic momentum} methods have been widely adopted in training deep neural networks. 
However, their theoretical analysis of convergence of the training objective and the generalization error for prediction is still under-explored.
This paper aims to bridge the gap between practice and theory by analyzing the stochastic gradient (SG) method, and the stochastic momentum methods including two famous variants, \textit{i.e.}, the stochastic heavy-ball (SHB) method and the stochastic variant of Nesterov's accelerated gradient (SNAG) method.
We propose a framework that unifies the {\emph{three}} variants. 
We then derive the convergence rates of the norm of gradient for the non-convex optimization problem, and analyze the generalization performance through the uniform stability approach.
Particularly, the convergence analysis of the training objective exhibits that SHB and SNAG have no advantage over SG. 
However, the stability analysis shows that the momentum term can improve the stability of the learned model and hence improve the generalization performance. 
These theoretical insights verify the common wisdom and are also corroborated by our empirical analysis on deep learning.

\end{abstract}

\section{Introduction}
Momentum methods have a long history dating back to 1960's. 
\cite{poly64} proposed a heavy-ball (HB) method that uses the previous two iterates when computing the next one. 
The original motivation of momentum methods is to speed up the convergence for convex optimization. 
For a twice continuously differential strongly convex and smooth objective function, Polyak's analysis yields an accelerated linear convergence rate over the standard gradient method. 
In 1983, 
\cite{citeulike:9501961} proposed an accelerated gradient (NAG) method, which is also deemed as a momentum method and achieves the optimal $O(1/t^2)$ convergence rate for convex smooth optimization~\footnote{$t$ is the number of iterations.}, which has a clear advantage over standard gradient method with $O(1/t)$ convergence for the same problem. 
It was later on shown to have an accelerated linear convergence rate for smooth and strongly convex optimization problems~\cite{opac-b1104789}. 
Both the HB method and the NAG method use a {\emph{momentum}} term in updating the solution, i.e., the difference between current iterate and the previous iterate. 
Therefore, both methods have been referred to as {\emph{momentum}} methods in literature.  

Due to recently increasing interests in deep learning, the stochastic variants of HB and NAG methods have been employed broadly in optimizing deep neural networks~\cite{krizhevsky2012imagenet,citeulike:12721718}. 
\cite{DBLP:conf/icml/SutskeverMDH13} are probably the first to study SNAG and compare it with SHB for optimizing deep neural networks. 
Although they have some interesting findings of these two methods in deep learning (e.g., a distinct improvement in performance of SNAG is usually observed in their experiments), their analysis and argument are mostly restricted to convex problems.
Moreover, some questions remained unanswered, e.g., 
(i) Do SHB and SNAG enjoy faster convergence rates than SG for deep learning (non-convex optimization problems) as in the convex and deterministic setting?  
(ii) If not, what is the advantage of these two methods over SG?  
(iii) Why does SNAG often yield improvement over SHB?

In this paper, we propose and analyze a unified framework for stochastic momentum methods and stochastic gradient method aiming at bringing answers and more insights to above questions.  
We summarize our results and contributions as follows:
\mycomment{
(1) We present a unified stochastic momentum method with a free scalar parameter that includes SHB, SNAG and SG as special cases by setting three different values for the free parameter. 
(2) We present a unified convergence analysis of the gradient's norm of the training objective of these stochastic methods for non-convex optimization,  revealing the same rate of convergence for minimizing the training objective function.
(3) We analyze the generalization performance of the unified framework through the uniform stability approach.
The result exhibits a clear advantage of stochastic momentum methods, i.e., adding a momentum helps generalization.
(4) Our empirical results for learning deep neural networks complete the unified view and analysis by showing that (i) there is no clear advantage of SHB and SNAG over SG in convergence speed of the training error; (ii) the advantage of SHB and SNAG lies at better generalization due to more stability; (iii)  SNAG usually achieves the best tradeoff between speed of convergence in training error and  stability of testing error among the three stochastic methods.
}

\begin{itemize}
\item We propose a unified stochastic momentum framework parameterized by a free scalar parameter.
      The framework reduces to SHB, SNAG and SG by setting three different values for the free parameter. 

\item We present a unified convergence analysis of the gradient's norm of the training objective of these stochastic methods for non-convex optimization, revealing the same rate of convergence for three variants.

\item We analyze the generalization error of the unified framework through the uniform stability approach.
The result exhibits a clear advantage of stochastic momentum methods, i.e., adding a momentum helps generalization.

\item Our empirical results for learning deep neural networks complete the unified view and analysis by showing that (i) there is no clear advantage of SHB and SNAG over SG in convergence speed of the training error; (ii) the advantage of SHB and SNAG lies at better generalization due to more stability; (iii)  SNAG usually achieves the best tradeoff between speed of convergence in training error and  stability of testing error among the three stochastic methods.
\end{itemize}

\section{More Related Work}
There is much analysis on the momentum methods for deterministic optimization. 
Nesterov pioneered the work of accelerated gradient methods for smooth convex optimization~\cite{opac-b1104789}. 
The convergence analysis of HB has been recently extended to smooth functions for both convex~\cite{arxiv1412,DBLP:journals/jmiv/OchsBP15} and non-convex deterministic optimization~\cite{DBLP:journals/siamis/OchsCBP14,DBLP:journals/corr/Ochs16}. 
As the rising popularity of deep neural networks, the stochastic variants of HB and NAG have been employed widely for training neural networks and leading to tremendous success for many problems in computer vision and speech recognition~\cite{krizhevsky2012imagenet,citeulike:12721718,DBLP:conf/icml/SutskeverMDH13}.
However, their {\emph{stochastic}} variants in non-convex optimization are under-explored.

It is worth mentioning that two recent works have established the convergence results of the SG method~\cite{DBLP:journals/siamjo/GhadimiL13a} and the stochastic version of a different variant of accelerated gradient method for non-convex optimization~\cite{DBLP:journals/mp/GhadimiL16}. 
However, the variant of accelerated gradient method in~\cite{DBLP:journals/mp/GhadimiL16} is hard to be explained in the framework of momentum methods and is not widely employed for optimizing deep neural networks. 
Moreover, their analysis is not applicable to the SHB method. 
Hence, from a theoretical standpoint, it is still interesting to analyze the stochastic variants of the Nesterov's accelerated gradient method and the HB method for stochastic non-convex optimization, which are extensively employed for learning deep neural networks. 
Our unified analysis shows that they enjoy the same order of convergence rate as the SG method, which conincides with the results in~\cite{DBLP:journals/siamjo/GhadimiL13a,DBLP:journals/mp/GhadimiL16}. 

On the other hand, there exist few studies on analyzing the statistical properties (e.g., the generalization error) of the model learned by the SG method or stochastic momentum methods for minimizing the empirical risk. 
Conventional studies on the SG method in terms of statistical property focus on one pass learning, i.e., the training examples are passed once~\cite{Cesa-BianchiCG04}. 
Recently, there emerge several works that aim to establish the statistical properties of the multiple pass SG methods in machine learning~\cite{DBLP:conf/nips/LinR16,hardticml2016train}. 
The latter work is closely related to the present work, which established the generalization error of the SG method with multiple pass for both convex and non-convex learning problems by analyzing the uniform stability. 
Nevertheless, it remains an open problem from a theoretical standpoint how the momentum term helps improve the generalization, though it has been observed to yield better performance in practice for deep learning~\cite{DBLP:conf/icml/SutskeverMDH13}.  
Our unified analysis of the uniform stability of the SG method and stochastic momentum methods explicitly exhibit the advantage of the stochastic momentum methods in terms of the generalization error, hence providing the theoretical support for the common wisdom.



In the remainder of the paper, we first review the HB and NAG method, and present their stochastic variants.
Then we present a unified view of these momentum methods.
Next, we present the convergence and generalization analysis for stochastic momentum methods.
In addition, we present empirical results for comparing different methods for optimizing deep neural networks.
Finally, we conclude this work.

\section{Momentum Methods And Their Stochastic Variants}
\subsection{Notations and Setup}
Let us consider a general setting of learning with deep learning as a special case. 
Given a set of training examples $\mathcal S=(\data_1, \ldots, \data_n)$ sampled from an unknown distribution $\mathcal D$, the goal of learning is to find a model $\x$ that minimizes the population risk, i.e.,
\begin{align}\label{eqn:population}
\min_{\x\in\Omega} F(\x) \triangleq \E_{\data \sim\mathcal D}[\ell(\x, \data)],
\end{align}
where $\ell$ is a loss function,  $\ell(\x, \data)$ denotes the loss of the model $\x$ on the example $\data$ and $\Omega$ denotes the hypothesis class of the model. Since we cannot compute $F(\x)$ due to unknown distribution $\mathcal D$, one usually learns a model by minimizing the empirical risk, i.e.,
\begin{align}\label{eqn:erm}
\min_{\x\in\Omega} f(\x)\triangleq \frac{1}{n}\sum_{i=1}^n\ell(\x, \data_i)  .
\end{align}

Two concerns usually present in the above empirical risk minimization approach.
First, how fast the optimization algorithm solves Problem~(\ref{eqn:erm}).
This is usually measured by the speed of convergence to the optimal solution.
However, it is NP-hard to find the global optimal solution for a general non-convex optimization problem~\cite{Hillar:2013:MTP:2555516.2512329}. 
As with many previous works~\cite{DBLP:journals/siamjo/GhadimiL13a,DBLP:journals/mp/GhadimiL16,DBLP:journals/corr/ReddiHSPS16,DBLP:journals/corr/ZhuH16a}, we study the convergence rate of an iterative algorithm to the critical point, i.e., a point $\x_*$ such that $\nabla f(\x_*)=0$.

\mycomment{
There are two important questions about the above empirical risk minimization approach: 
(i) how fast is an optimization algorithm for solving~(\ref{eqn:erm})? 
(ii) how good is the learned model in terms of generalization performance? 
For the first question, one usually measures the speed of optimization by the convergence to the optimal solution.
However, it is an NP-hard problem to find the global optimal solution for a general non-convex optimization problem~\cite{Hillar:2013:MTP:2555516.2512329}. 
As with many previous works~\cite{DBLP:journals/siamjo/GhadimiL13a,DBLP:journals/mp/GhadimiL16,DBLP:journals/corr/ReddiHSPS16,DBLP:journals/corr/ZhuH16a}, we study the convergence rate of an iterative algorithm to the critical point, i.e., a point $\x_*$ such that $\nabla f(\x_*)=0$.
}

Second, how the model learned by solving Problem~(\ref{eqn:erm}) generalizes to different data.
It is usually measured by the population risk $F(\xh)$ defined in~(\ref{eqn:population}). 
Since the model $\xh$ is learned from the random samples $\data_1, \ldots, \data_n$ with randomness in the optimization algorithm itself, the expected population risk $\E[F(\xh)]$ is also used for the analysis with the expectation taking over the randomness in the samples and the algorithm itself. 
One way to assess the expected population risk is the generalization error, i.e., the difference between the population risk and the empirical risk,
\begin{align}
    \epsilon_{\text{gen}} \triangleq \E[F(\xh) - f(\xh)].
\end{align}


We use $\nabla h(\x)$ to denote the gradient of a smooth function. 
A function is smooth iff there exists $L>0$ such that
\begin{align}
\|\nabla h(\y) - \nabla h(\x)\|\leq L\|\y - \x\|, \quad \forall \x, \y\in\R^d  ,
\end{align}
\noindent
where $\|\cdot\|$ denotes the Euclidean norm. 
Note that the above inequality does not imply convexity. 
Through the paper, we assume that $\ell(\x, \q)$ a $G$-Lipschitz continuous and $L$-smooth non-convex function in $\x$, and assume that $\Omega=\R^d$. 
It follows that $f(\x)$ is $G$-Lipschitz continuous and $L$-smoothness. 

\subsection{Stochastic Momentum Methods}

We denote by $\G_k = \G(\x_k;\xi_k)$ a stochastic gradient of $f(\x)$ at $\x_k$ depending on a random variable $\xi_k$ such that  $\E[\G(\x_k ;\xi_k)]=\nabla f(\x_k)$. 
In the context of the empirical risk minimization~(\ref{eqn:erm}), $\G(\x_k; \xi_k) = \nabla\ell(\x_k; \data_{i_k})$, where $i_k$ is a random index sampled from $\{1,\ldots, n\}$.

There are two variants of momentum methods for solving~(\ref{eqn:erm}), i.e., HB and NAG.
HB was originally proposed for optimizing a smooth and strongly convex objective function.
Based on HB, the update of stochastic HB (SHB) is given below for $k=0, \ldots, $
\begin{align}\label{eqn:shb1}
  \text{SHB:} \quad \x_{k+1} = \x_k - \alpha \G(\x_k; \xi_k) + \beta(\x_k - \x_{k-1})   ,
\end{align}
with $\x_{-1} = \x_0$, where $\beta\in[0,1)$ is the momentum constant and $\alpha$ is the step size. 
Equivalently, the above update can be implemented by the following two steps for $k=0, \ldots, $:
\begin{equation}\label{eqn:shb2}
  \text{SHB:}\quad
  \left\{ 
    \begin{aligned}
      \v_{k+1} & = \beta \v_{k} - \alpha \G(\x_k; \xi_{k})\\
      \x_{k+1} & = \x_k + \v_{k+1} .
    \end{aligned}
  \right.
\end{equation}

Based on NAG~\cite{opac-b1104789}, the update of stochastic NAG (SNAG) consists of the two steps below for $k=0, \ldots,$:
\begin{equation}\label{eqn:sag1}
\text{SNAG:}\quad 
  \left\{
    \begin{aligned}
      \y_{k+1} & = \x_k - \alpha \G(\x_k; \xi_{k})\\
      \x_{k+1} & = \y_{k+1} + \beta(\y_{k+1} - \y_k)  ,
    \end{aligned}
  \right.
\end{equation}
with $\y_0 = \x_0$.
By introducing $\v_k = \y_k - \y_{k-1}$ with $\v_0 =0$, the above update can be equivalently written as
\begin{equation}\label{eqn:sag2}
\text{SNAG:}\quad 
  \left\{
    \begin{aligned}
      \v_{k+1}& = \beta\v_k - \alpha \G(\y_k + \beta \v_k; \xi_k)\\
      \y_{k+1}& =\y_k + \v_{k+1}  .
    \end{aligned}
  \right.
\end{equation}
Finally, the traditional view  of SG can be written as
\begin{equation}\label{eqn:SG}
  \text{SG:}\quad\quad \x_{k+1} = \x_k - \alpha\G(\x_k; \xi_k)  .
\end{equation}
By comparing~(\ref{eqn:sag2}) to~(\ref{eqn:shb2}), one might argue that the difference between HB and NAG lies at the point for evaluating the gradient~\cite{DBLP:conf/icml/SutskeverMDH13}. 
We will present a different unified view of the three methods that allows us to analyze them in a unified framework.
The convergence of HB and NAG has been established for convex optimization~\cite{poly64,citeulike:9501961,opac-b1104789,arxiv1412,DBLP:journals/jmiv/OchsBP15}.

\mycomment{
The requirement of HB and NAG makes them prohibitive in big data.
Therefore, when employed in the optimization of deep neural networks, the (sub)gradient is usually replaced with the stochastic (sub)gradient, which yields the stochastic variants.
We denote by $\G_k=\G(\x_k;\xi_k)$ a stochastic gradient of $f(\x)$ at $\x_k$ depending on a random variable $\xi_k$ such that  $\E[\G(\x_k ;\xi_k)]=\nabla f(\x_k)$. 
In the context of the empirical risk minimization~(\ref{eqn:erm}), $\G(\x_k; \xi_k) = \nabla\ell(\x_k; \data_{i_k})$, where $i_k$ is a random index sampled from $\{1,\ldots, n\}$.
Then the update of stochastic HB (SHB) becomes
\begin{align}\label{eqn:shb}
\text{SHB:} \quad \x_{k+1} = \x_k - \alpha \G(\x_k; \xi_k) + \beta(\x_k - \x_{k-1})   .
\end{align}
The update of stochastic NAG (SNAG) becomes
\begin{equation}\label{eqn:sag}
\text{SNAG:}\quad \left\{\begin{aligned}
 \y_{k+1}& = \x_k - \alpha \G(\x_k; \xi_k)\\
\x_{k+1}&  =\y_{k+1} + \beta(\y_{k+1} - \y_k)   .
\end{aligned}\right.
\end{equation}
Finally, the traditional view  of SG can be written as
\begin{equation}\label{eqn:SG}
  \text{SG:}\quad\quad\x_{k+1} = \x_k - \alpha\G(\x_k; \xi_k)  .
\end{equation}
}

\section{A Unified View of Stochastic Momentum Methods}
In this section, we present a unified view of the two (stochastic) momentum methods and (stochastic) gradient methods.
We first present the unified framework and then show that HB, NAG and the gradient method are special cases of the unified framework. 
Denote by $\G(\x_k)$ either a gradient or a stochastic gradient of $f(\x)$ at $\x_k$.

Let $\alpha>0$, $\beta\in[0,1)$, and  $s\geq 0$.
The updates of the stochastic unified momentum (SUM) method  are given by
\begin{equation}
\label{eqn:um}
  \text{SUM}:\quad
    \left\{
      \begin{aligned}
        \y_{k+1} &  = \x_k - \alpha \G(\x_k)\\
        \y^s_{k+1} & = \x_k - s\alpha \G(\x_k)\\
        \x_{k+1} & = \y_{k+1} + \beta(\y^s_{k+1} - \y^s_k)  ,
      \end{aligned}
    \right.
\end{equation}
for $k\geq 0$ with $\y^s_0 = \x_0$. 
It is notable that in the update of $\x_{k+1}$, a momentum term is constructed based on the auxiliary sequence $\{\y_k^s\}$, whose update is parameterized by $s$.
The following proposition indicates that SUM reduces to the concerned three special cases by setting different values to $s$.

\begin{prop}
\label{proposition:special_cases}
SUM~(\ref{eqn:um}) reduces to the three variants SG~(\ref{eqn:SG}), SHB~(\ref{eqn:shb1}) and SNAG~(\ref{eqn:sag1}) by setting $s = \frac{1}{1 - \beta}$, $s = 0$ and $s = 1$, respectively.
Particularly, the update of SG is
$
\x_{k+1} = \x_{k} - \frac{\alpha}{1 - \beta} \G(\x_{k}, \xi_{k}) ,
$
where the step size is $\frac{\alpha}{1 - \beta}$.
\end{prop}

From the above result, we can see that SHB, SNAG and SG are three variants of SUM.
Moreover, the SUM view of SG implies that  SG can have a larger ``effective'' step size (i.e., $\alpha/(1-\beta)$) before the gradient $\G(\x_k)$ than that of SHB and SNAG.
We note that this is a very important observation about SG since setting a smaller effective step size for SG (e.g., the same as that in SNAG) will yield much worse convergence of training error as observed in experiments.

To facilitate the unified analysis of the stochastic momentum methods, we note that~(\ref{eqn:um}) implies the recursions in~(\ref{eqn:rec}) and~(\ref{eqn:rec2}) given in the following lemma.
\begin{lemma}\label{lem:k} 
Let $\p_k$ be
\begin{equation}\label{eqn:p}
\p_k = \left\{
\begin{aligned}
&\frac{\beta}{1-\beta}(\x_k - \x_{k-1} + s\alpha \G(\x_{k-1})),\: k\geq 1\\
& 0,\quad k=0\\
\end{aligned}\right.
\end{equation}
and
\begin{equation}\label{eqn:v}
\v_k = \frac{(1-\beta)}{\beta}\p_k   .
\end{equation}
Then for any $k\geq 0$, we have
\begin{align}
\label{eqn:rec}
\x_{k+1} + \p_{k+1} & = \x_k + \p_k  - \frac{\alpha}{ 1- \beta}\G(\x_k)   ,   \\
\v_{k+1} & = \beta \v_k + ((1-\beta)s - 1) \alpha \G(\x_k)  \label{eqn:rec2}   .
\end{align}
\end{lemma}
{\bf Remark:} 
We note that a similar recursion in~(\ref{eqn:rec}) with $s=0$ and $s=1$ has been observed and employed to~\cite{arxiv1412} for deterministic convex optimization. 
However, the recursion in~(\ref{eqn:rec2}) for $\v_k$ (i.e., $\p_k$)  is a key to our convergence analysis for non-convex optimization and importantly the generalization to any $s$ allows us to analyze SHB, SNAG and SG in a unified framework.

Finally, we present a lemma stating the cumulative effect of updates for each iterate, which will be useful for our generalization error analysis.
\begin{lemma}\label{lem:cm}
Given the update in~(\ref{eqn:um}), for any $k\geq 0$ we have
\begin{align}
    \x_{k+1}  = \x_{0} - \sum_{\tau=0}^{k} \left\{ \frac{1}{1-\beta} - \beta^{k-\tau+1}\frac{1-s(1-\beta)}{1-\beta} \right\}\alpha \grad(\x_{\tau})  .
\end{align}
\end{lemma}
{\bf Remark:} The above cumulative update reduce to the following three cases for SHB ($s=0$), SNAG ($s=1$) and SG ($s=1/(1-\beta)$).
\begin{align*}
    \x_{k+1}=\left\{\begin{array}{ll}\x_0 -\sum_{t=0}^{k} \left\{ \frac{1}{1-\beta} - \frac{\beta^{k-\tau+1}}{1-\beta}\right\}\alpha \grad(\x_{\tau})&s=0\\
    \x_0 -\sum_{t=0}^{k} \left\{ \frac{1}{1-\beta} - \frac{\beta^{k-\tau+2}}{1-\beta}\right\}\alpha \grad(\x_{\tau})&s=1\\
    \x_0 -\sum_{t=0}^{k} \frac{1}{1-\beta} \alpha \grad(\x_{\tau})&s=\frac{1}{1-\beta}\end{array}\right.
\end{align*}
From the above cumulative update, we can see that SHB and SNAG have smaller step size for each stochastic gradient. This is the main reason that SHB and SNAG are more stable than SG, and hence yield a solution with better generalization performance. We will present a more formal analysis of generalization error later.


\section{Convergence Analysis of SUM}
\label{sec:conv}
In this section, we present the convergence results for the empirical risk minimization~(\ref{eqn:erm}) of the SUM methods.
As mentioned before, for deep learning problems the loss function $\ell(\x, \data)$ is a non-convex function, which makes finding the global optimal solution an NP-hard problem. Instead, as in many previous works we will present the convergence rates of SUM in terms of the norm of the gradient. We will present the main results first and then sketch the analysis. Detailed proofs are deferred to the supplement due to limit of space.  

\subsection{Main results}


\begin{thm}~\label{thm:3}(Convergence of SUM)
Suppose $f(\x)$ is a non-convex and $L$-smooth function, $\E[\|\G(\x; \xi) - \nabla f(\x)\|^2]\leq \sigma^2$ and $\|\nabla f(\x)\|\leq G$ for any $\x$.
 Let update~(\ref{eqn:um}) run for $t$ iterations with $\G(\x_k; \xi_k)$. By setting $\alpha =\min\{\frac{1-\beta}{2L}, \frac{C}{\sqrt{t+1}}\}$ we have
\begin{align*}
&\min_{k=0,\ldots, t}\E[\|\nabla f(\x_k)\|^2]\\
&\leq \frac{2(f(\x_0) - f_*)(1-\beta)}{t+1}\max\left\{\frac{2L}{1-\beta}, \frac{\sqrt{t+1}}{C}\right\} \\
&+ \frac{C}{\sqrt{t+1}}\frac{L\beta^2((1-\beta)s -1)^2(G^2+\sigma^2) + L\sigma^2(1-\beta)^2}{(1-\beta)^3} \end{align*}
\end{thm}
{\bf Remark:}
We would like to make several remarks. 
(i) The assumption on the magnitude  of the gradient  and the variance of stochastic gradient can be simply replaced by the magnitude of the stochastic gradient, which are standard assumptions in the previous analysis of stochastic gradient method~\cite{DBLP:journals/siamjo/GhadimiL13a}. 
(ii) This is the first time that the convergence rate of SHB for non-convex optimization is established. A similar convergence rate of SG and a different stochastic variant of accelerated gradient method has been established in~\cite{DBLP:journals/siamjo/GhadimiL13a} and~\cite{DBLP:journals/mp/GhadimiL16}, respectively under similar assumptions. 
(iii) The unified convergence makes it clear that the difference of the convergence bounds for different variants of SUM lies at the term $((1-\beta)s-1)^2$, which is equal to $\beta^2$, $\beta^4$ and $0$ for SHB, SNAG and SG, respectively. 
(iv) The step size $\alpha$ of different variants of SUM used in the analysis of Theorem~\ref{thm:3} is the same value.

The above result shows that the convergence upper bound of the three methods are of the same order, i.e., $O(1/\sqrt{t})$ for the gradient's square norm. In addition, when the momentum term $\beta$ is large, the effect of  different values of $s$ in the term would $L\beta^2((1-\beta)s -1)^2(G^2+\sigma^2)$  becomes marginal in contrast to the term $L\sigma^2(1-\beta)^2$ in the convergence bound. This reveals that SHB and SNAG have no advantage over SG in terms of empirical risk minimization.

Below, we present a result with different step sizes $\alpha$ for different variants of SUM in the analysis, which sheds more insights of different methods.
\begin{thm}~\label{thm:4}(Convergence of SUM)
Suppose $f(\x)$ is a non-convex and $L$-smooth function, $\E[\|\G(\x; \xi) - \nabla f(\x)\|^2]\leq \sigma^2$ and $\|\nabla f(\x)\|\leq G$ for any $\x$.
Let update~(\ref{eqn:um}) run for $t$ iterations with $\G(\x_k; \xi_k)$. By setting $\alpha =\min\{\frac{1-\beta}{2L[1+((1-\beta)s-1)^2]}, \frac{C}{\sqrt{t+1}}\}$ we have
\begin{align*}
&\min_{k=0,\ldots, t}\E[\|\nabla f(\x_k)\|^2]\leq \frac{2(f(\x_0) - f_*)(1-\beta)}{t+1}\Lambda\\
&+ \frac{C}{\sqrt{t+1}}\frac{L\beta^2(G^2+\sigma^2) + L\sigma^2(1-\beta)^2}{(1-\beta)^3} \end{align*}
where  $
\Lambda = \max\left\{\frac{2L[1+((1-\beta)s-1)^2]}{1-\beta}, \frac{\sqrt{t+1}}{C}\right\}$.
\end{thm}
{\bf Remark: }
The above result allows us to possibly set a larger initial  value of $\alpha$ for SG (where $s=1/(1-\beta)$) and SNAG (where $s=1$) than that for SHB (where $s=0$). Our empirical studies for deep learning also confirms  this point.

\subsection{Generalization Error Analysis of SUM}
In this section, we provide a unified analysis for the generalization error of the solution returned by SUM after a finite number of iterations. By employing the unified analysis, we are able to analyze the effect of the scalar $s$ on the generalization error. Our analysis is inspired by~\cite{hardticml2016train}, which leverages the uniform stability of a randomized algorithm~\cite{Bousquet:2002:SG:944790.944801} to bound the generalization error of multiple pass SG method. To this end, we first introduce the uniform stability and its connection with generalization error.

Let $\mathcal A:\mathcal S\rightarrow \R^d$ denote a randomized algorithm that generates a model $\mathcal A(\mathcal S)$ from the set of training samples $\S$ of size $n$. The uniform stability measures that how likely the prediction of the learned model on any sample $\data$ would change if one example in $\S$ is changed to a different data. In particular, let $\S'$ denote a set of training examples that differ from $\S$ in one example. 
The algorithm $\A$ is said to be $\epsilon$-uniform stable, if
\begin{align*}
\epsilon(\mathcal A, n)\triangleq \sup_{\S,\S'}\sup_{\data}\E_{\mathcal A}[\ell(\mathcal A(\S), \data) - \ell(\mathcal A(\S'), \data)]\leq \epsilon
\end{align*}
The following proposition states that the generalization error of $\A(\mathcal S)$ is bounded by the uniform stability of $\A$.
\begin{prop}\label{prop:uniform_stable}
(Theorem 2.2 in~\cite{hardticml2016train})
For a randomized algorithm $A:\mathcal S\rightarrow\R$,
\begin{align}
    \E[F(\mathcal A(\S)) - f(\mathcal A(\S))]\leq \epsilon(\mathcal A, n)
\end{align}
\end{prop}
The above proposition allows us to use the uniform stability of a randomized algorithm as a proxy of the generalization error.
Below, we will show that SHB and SNAG are more uniform stable than SG, which exhibits that SHB and SNAG has potentially smaller generalization error than SG.

To proceed, we assume that loss function is $G$-Lipschitz continuous, then $|\ell(\mathcal A(\S), \data) - \ell(\mathcal A(\S'), \data)|\leq G\|\mathcal A(\S) - \mathcal A(\S')\|$ and $\epsilon(\A, n)\leq \sup_{\S,\S'}\E[\|\A(\S) - \A(\S')\|]$. To analyze the uniform stability of SUM, we will assume that there are two instances of SUM starting from the same initial solution, with one running on $\S$ and the other one running on $\S'$, where $\S$ and $\S'$ differs only at one example. Let $\x_t=\x_t(\S)$ denote the  $t$-th iterate of the first instance and $\x'_t=\x_t(\S')$ denote the $t$-th iterate of the second instance. Below, we establish a result showing how $\Delta_t=\E[\|\x_t - \x'_t\|]$ grows based on the unified framework in Lemma~\ref{lem:cm}.
\begin{prop}\label{prop:ge}
Assume that $\|\nabla\ell(\x, \data)\|_2\leq G$ for any $\x$ and $\data$ and $\ell(\x,\data)$ is $L$-smooth w.r.t $\x$. For two data sets $\S,\S'$ that differs at one example, let $\x_t$ and $\x'_t$ denote the $t$-th iterates of running SUM for the empirical risk minimization on $S$ and $S'$, we have
\[
\Delta_{t+1} \leq \sum_{k=0}^{t}\frac{2\alpha G}{n}\eta^t_k + \left(1-\frac{1}{n}\right)  \sum_{k=0}^{t}
\alpha L\eta_k^t \Delta_k
\]
with $\Delta_0=0$, where $
\eta_k^t = \frac{1}{1-\beta} - \beta^{t-k+1}\frac{1-s(1-\beta)}{1-\beta}$.
\end{prop}
{\bf Remark:} From the above result, we can easily analyze how the value of $s$ affects the growth of $\Delta_t$ that implies the growth of the generalization error of $\x_t$. The values of $\eta_k^t$ for the three variants (i.e., SHB, SNAG and SG) are given by $\eta_k^t(\text{SHB}) = \frac{1}{1-\beta} - \frac{\beta^{t-k+1}}{1-\beta}$, $\eta_k^t(\text{SNAG}) = \frac{1}{1-\beta} - \frac{\beta^{t-k+2}}{1-\beta}$ and $\eta_k^t(\text{SG}) = \frac{1}{1-\beta}$, respectively. It is obvious that $\eta_k^t(\text{SHB}) < \eta_k^t(\text{SNAG})< \eta_k^t(\text{SG})$. As a result, $\Delta_t$ of SG  grows faster than that of SNAG, and then followed by $\Delta_t$ of SHB. Since the generalization error of $\x_t$ is bounded by $\Delta_t$ up to a constant, we can conclude that by running the same number of iterations, the generalization error of the model returned by SHB and SNAG is potentially smaller than that of SG.

\paragraph{More Discussion.}
So far, we have analyzed the convergence rate for optimizing the empirical risk and the generalization error of the learned models of different variants of SUM, which provide answers to the questions raised at the beginning except the last one (why is SNAG observed to yield improvement on the prediction performance over SHB by some studies~\cite{DBLP:conf/icml/SutskeverMDH13}).
Next, we show that how our analysis can shed lights on this question.
In fact, the population risk of $\x_k$ that is usually assessed in practice by the testing error can be decomposed into three parts, consisting of the optimization error, the generalization error, and an optimization independent term, i.e.,
\begin{equation}\label{eqn:dec}
\E[F(\x_t)]=\E[f(\x_*)] +\underbrace{\E[F(\x_t) - f(\x_t)]}\limits_{\text{gen}}+ \underbrace{\E[f(\x_t) - f(\x_*)]}\limits_{\text{opt}}\notag
\end{equation}
where $\x_*$ is the optimal solution to the empirical risk minimization problem.
An informal analysis follows: our Theorem~\ref{thm:4} implies that SNAG converges potentially faster than SHB in terms of the optimization error, while Proposition~\ref{prop:ge} implies that SHB has potentially smaller generalization error. If the optimization error of SNAG decreases faster than the generalization error increases comparing with SHB, then SNAG could yield a solution with a smaller population risk.   
However, a rigorous analysis of the optimization error is complicated by the non-convexity of the problem. In next section, we will present empirical results to corroborate and complete our theoretical analysis.

\begin{figure}[t]
\centering
\subfigure[training]{\includegraphics[scale=0.24]{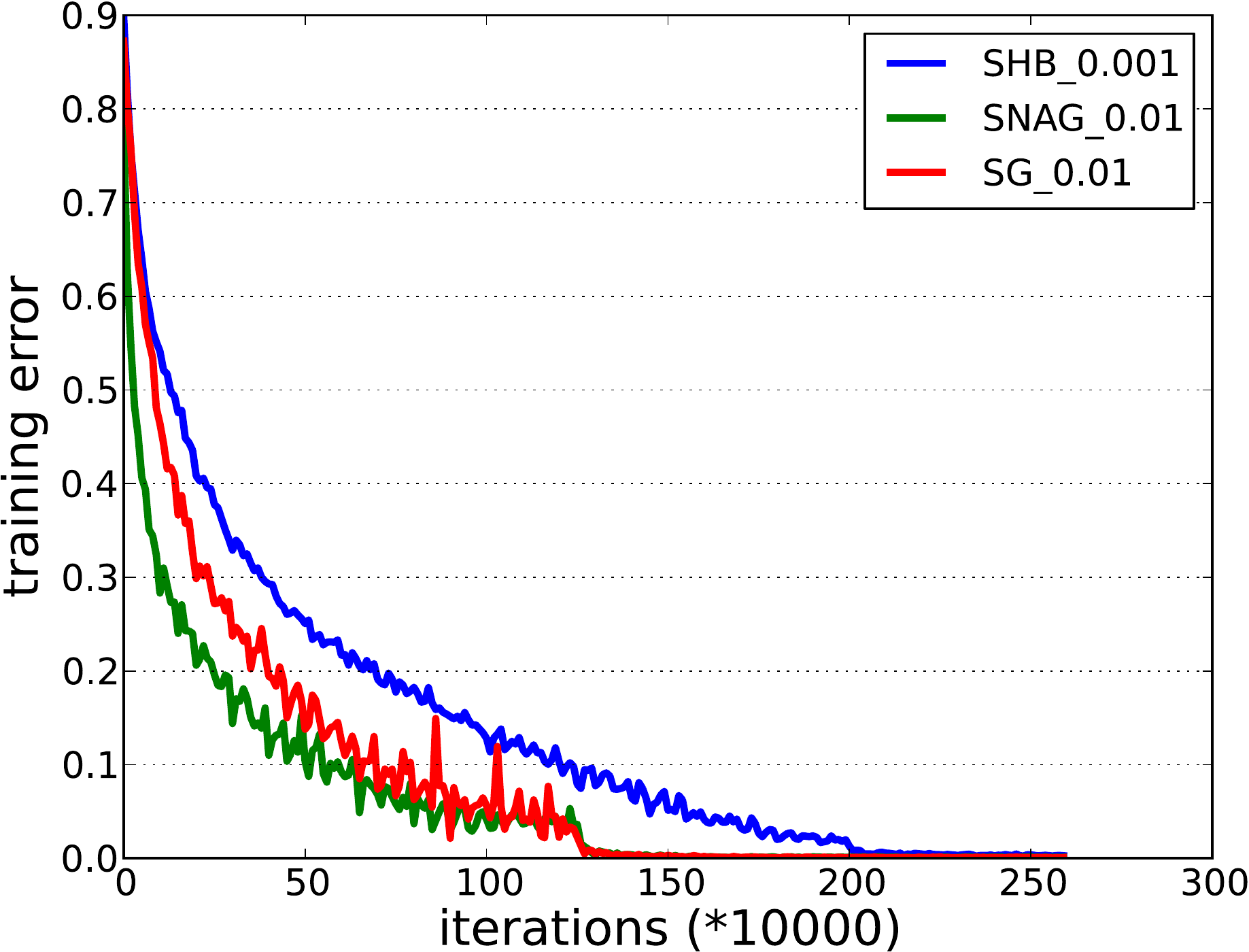}}\hspace*{-0.1in}
\subfigure[testing]{\includegraphics[scale=0.24]{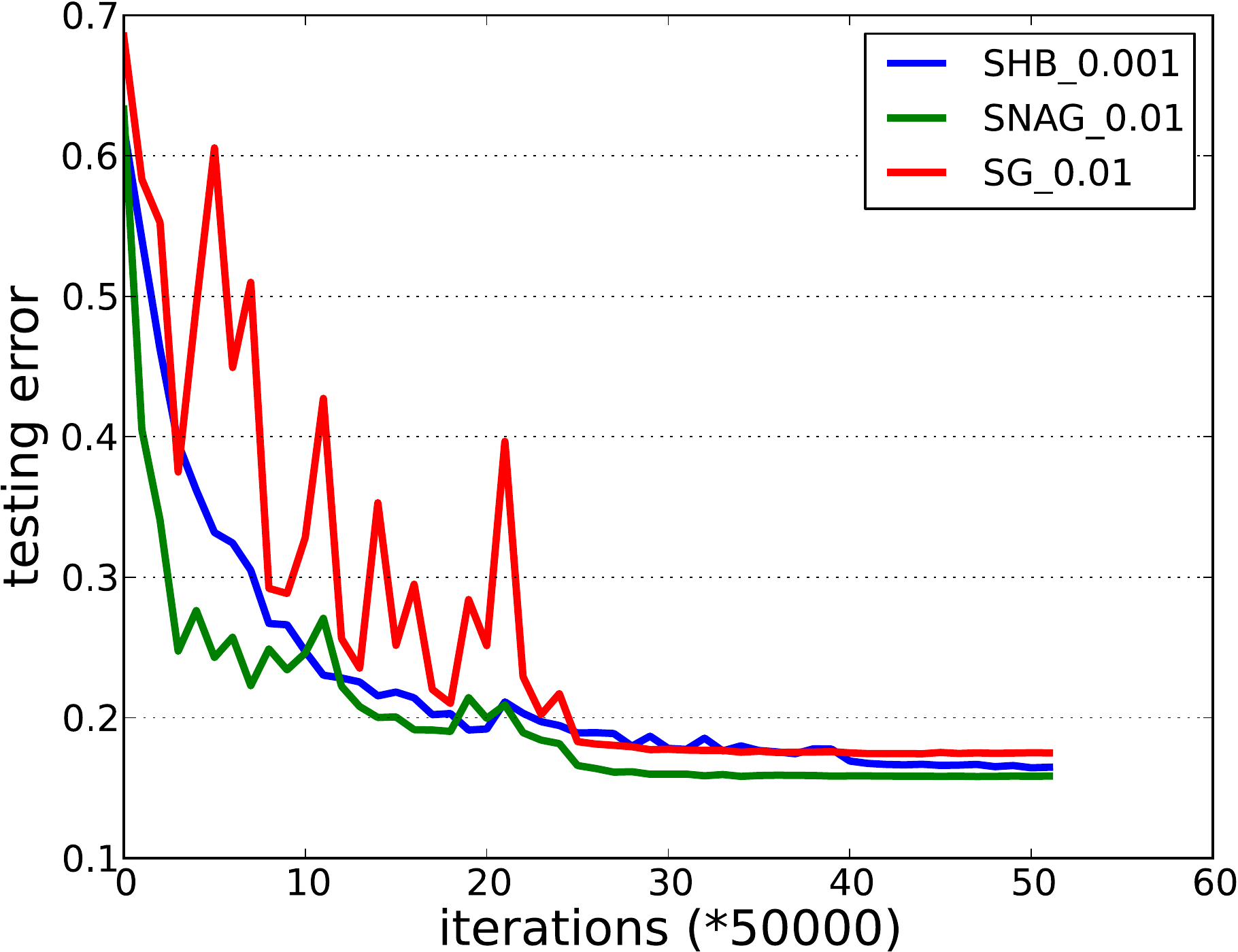}}\hspace*{-0.1in}
\caption{Training and testing error of different methods on CIFAR-10 with the best initial step size $\alpha$.
         The result is consistent with our convergence result in Theorem~\ref{thm:4}.}
\label{fig:best_init_step_size}
\end{figure}

\begin{figure*}[t]
\centering
  \subfigure[training error on CIFA-10]{
    \label{subfigure:train_error_cifar10}
    \includegraphics[scale=0.29]{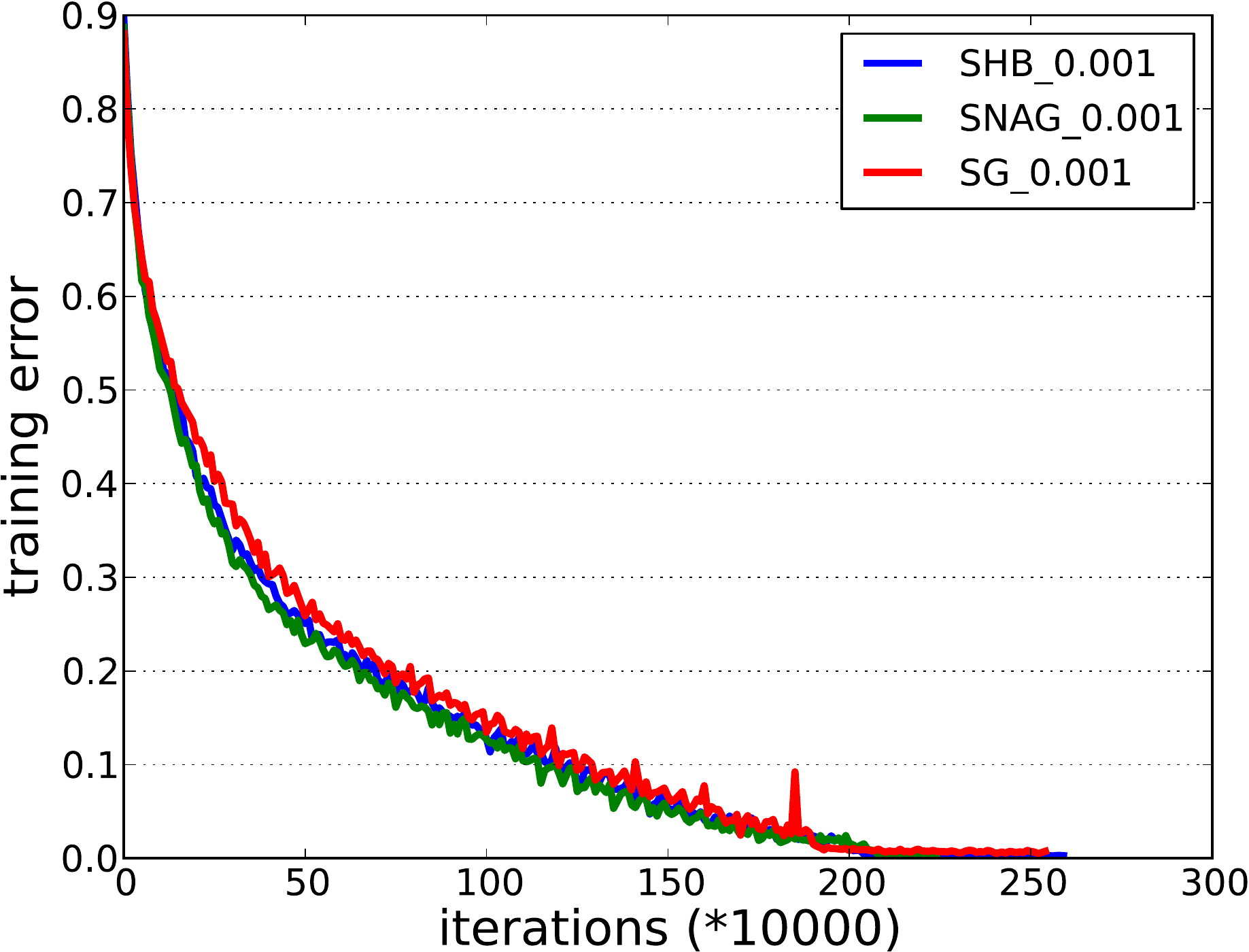}
  }
  \subfigure[testing error on CIFA-10]{
    \label{subfigure:test_error_cifar10}
    \includegraphics[scale=0.29]{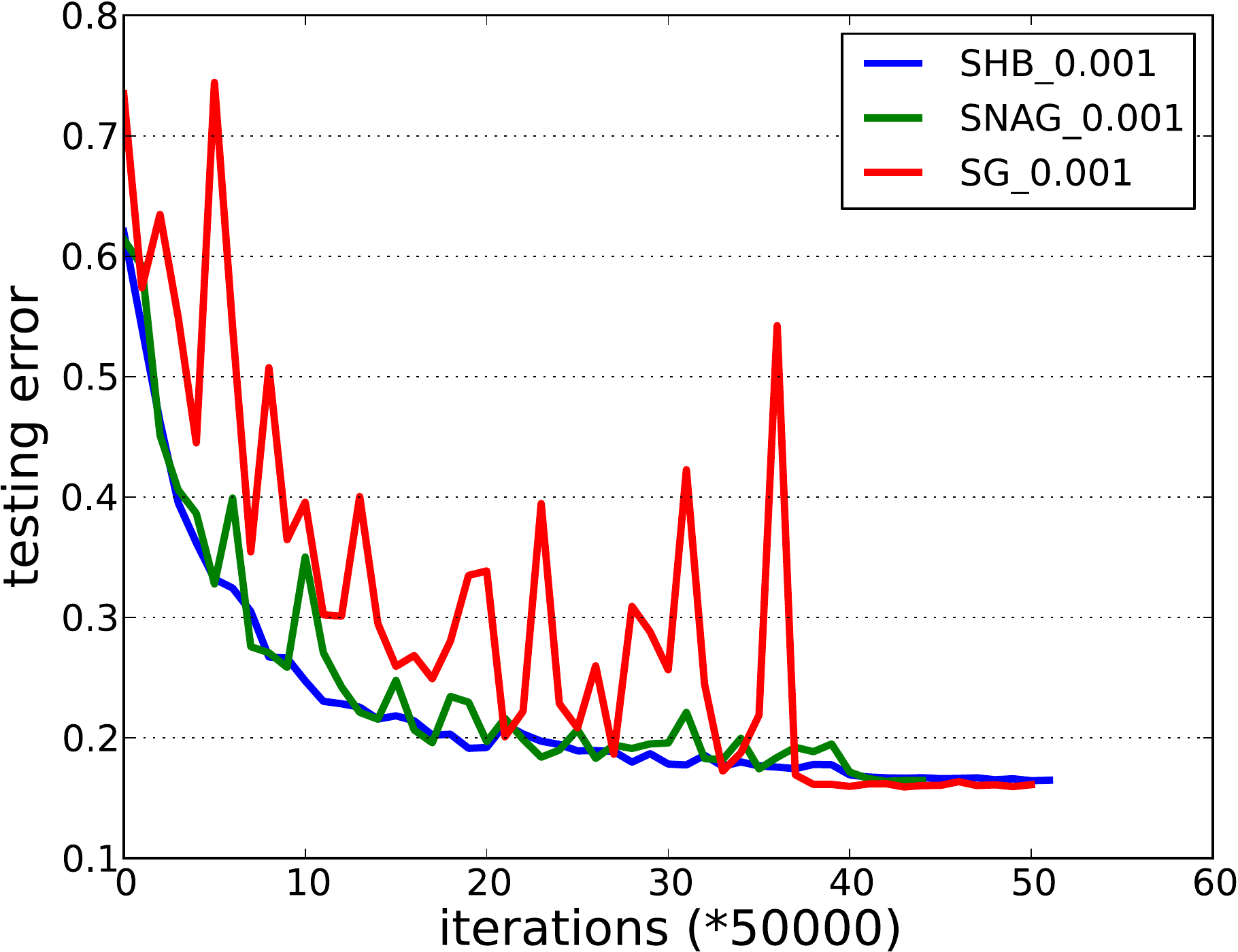}
  }
  \subfigure[abs difference on CIFAR-10]{
    \label{subfigure:abs_diff_cifar10}
    \includegraphics[scale=0.29]{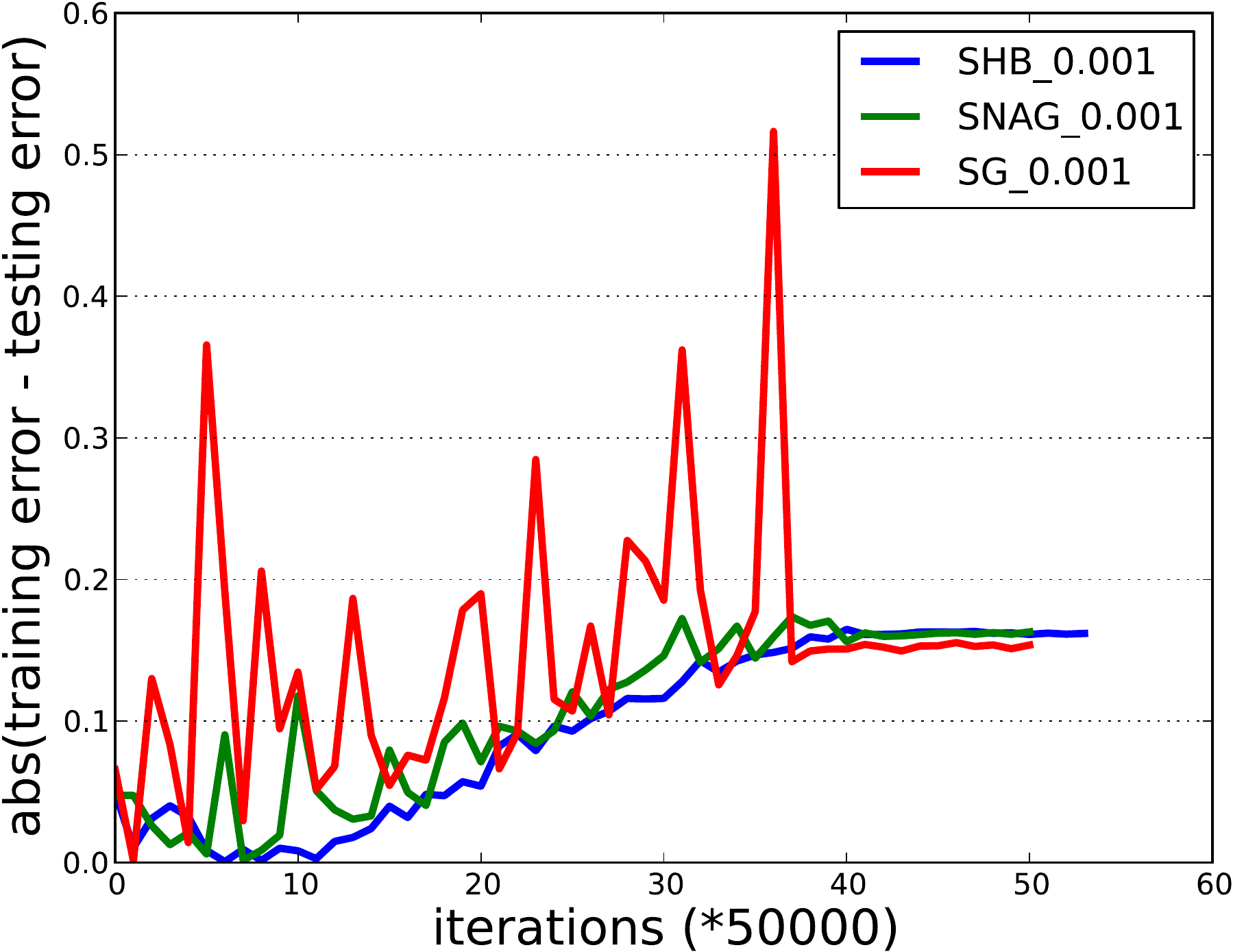}
  }
  
  \subfigure[training error on CIFA-100]{
    \label{subfigure:train_error_cifar100}
    \includegraphics[scale=0.29]{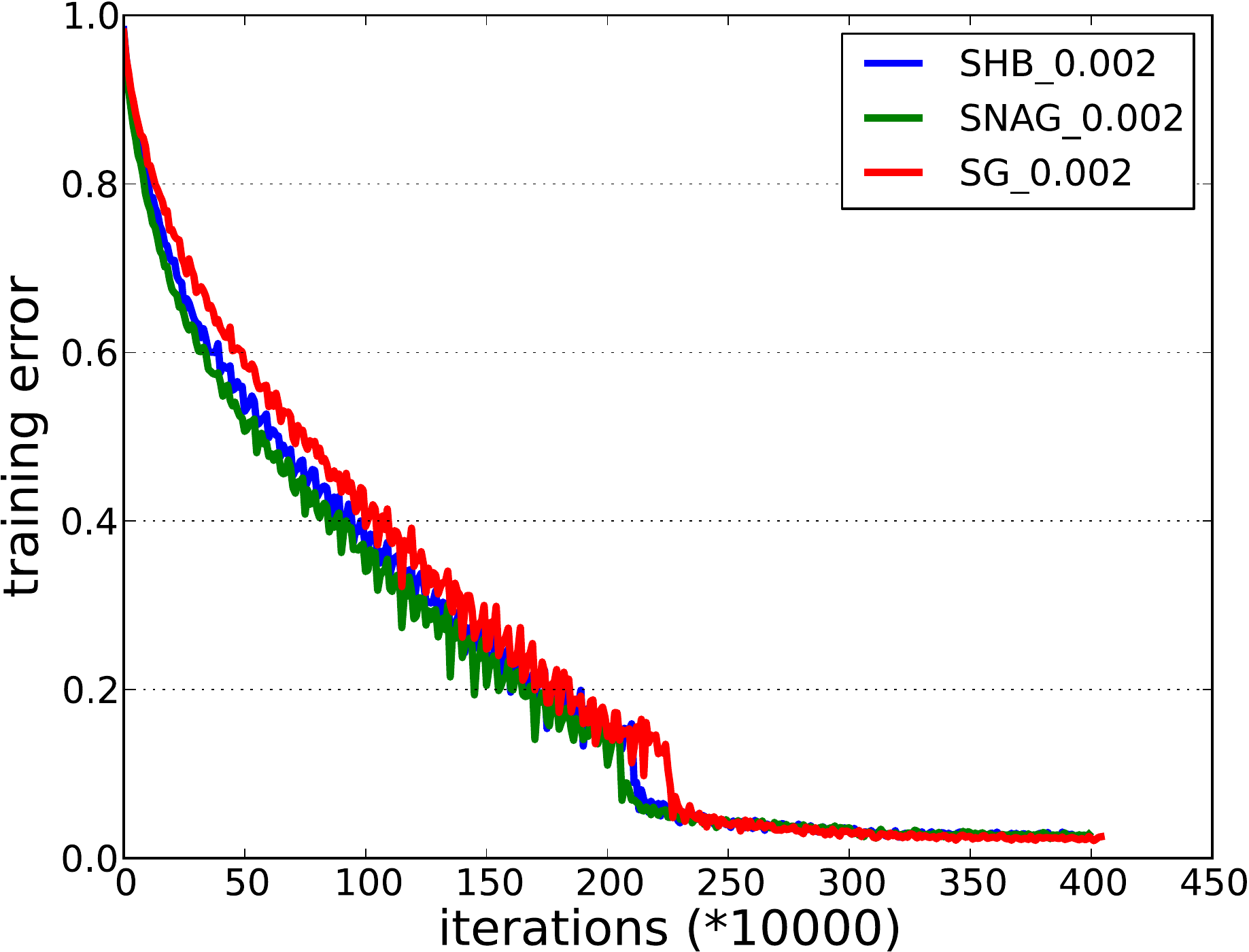}
  }
  \subfigure[testing error on CIFA-100]{
    \label{subfigure:test_error_cifar100}
    \includegraphics[scale=0.29]{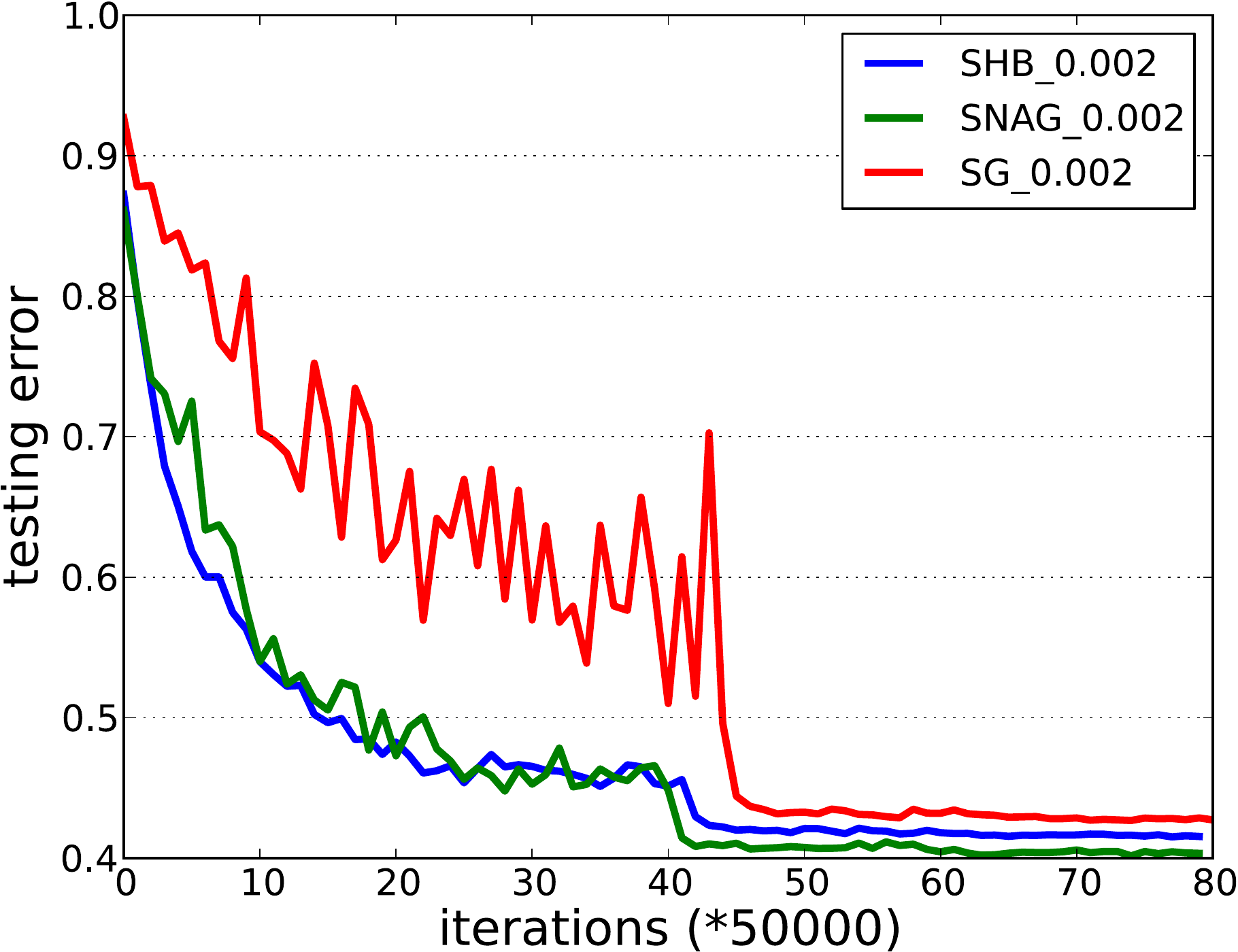}
  }
  \subfigure[abs difference on CIFAR-100]{
    \label{subfigure:abs_diff_cifar100}
    \includegraphics[scale=0.29]{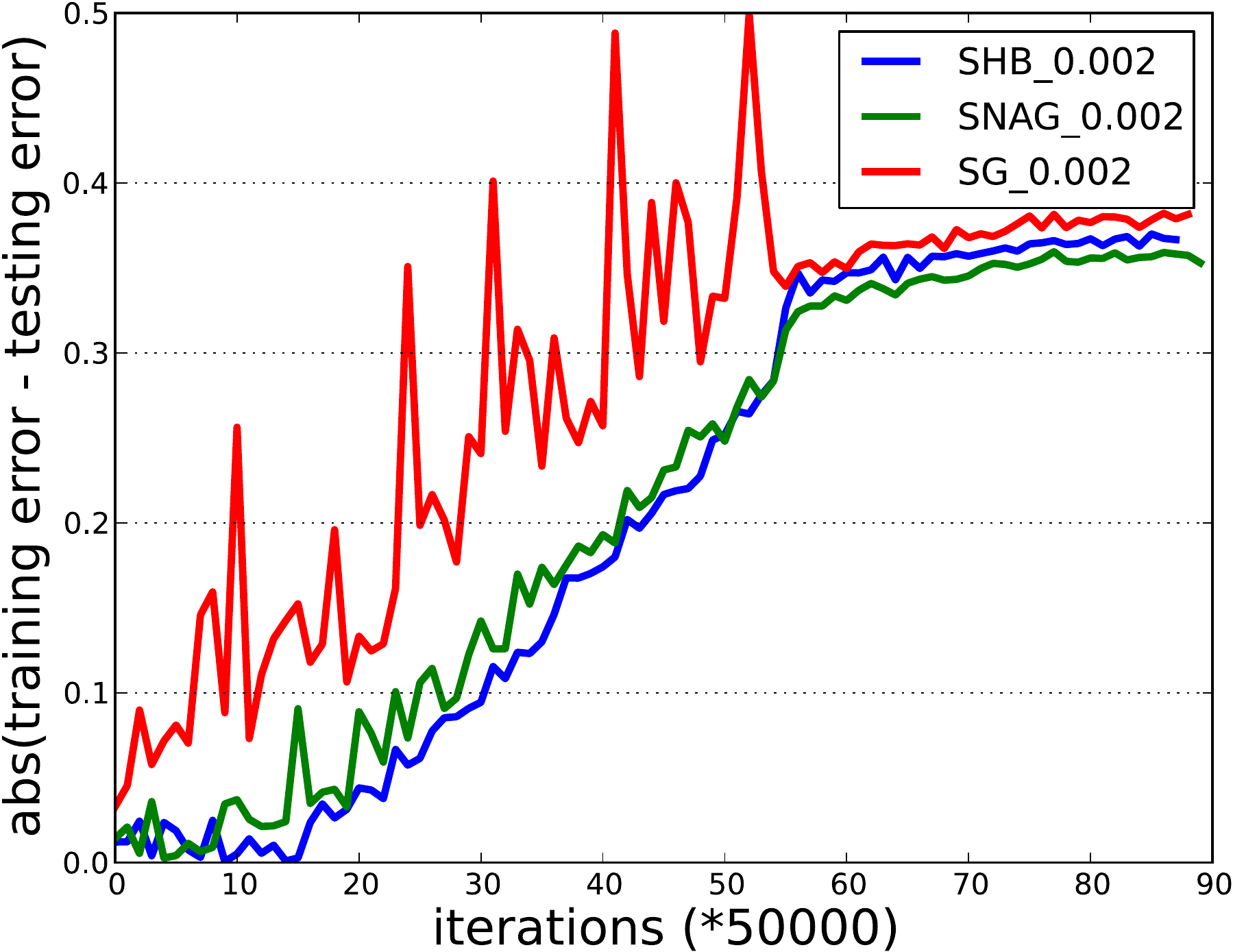}
  }
\caption{
         Training error, testing error and their absolute difference (i.e., $| \text{training error} - \text{testing error} |$) on CIFAR-10 and CIFAR-100 of three SUM variants.
         The numbers in legends indicate the initial step size $\alpha$.
         }
\label{fig:same_init_step_size}
\end{figure*}

\begin{figure*}[t]
\centering
  \subfigure[training error]{
    \label{subfigure:train_error_different_s}
    \includegraphics[scale=0.29]{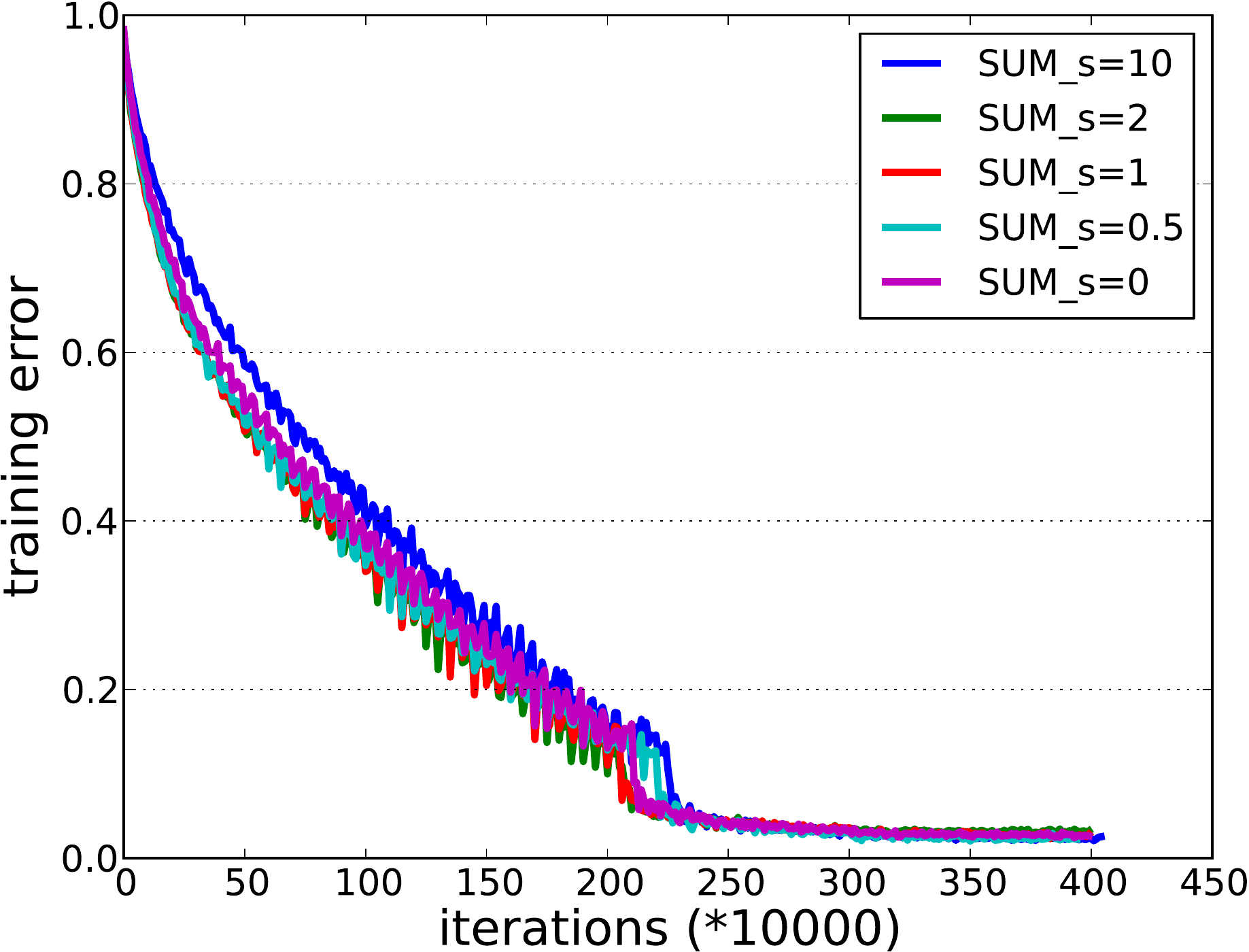}
  }
  \subfigure[testing error]{
    \label{subfigure:test_error_different_s}
    \includegraphics[scale=0.29]{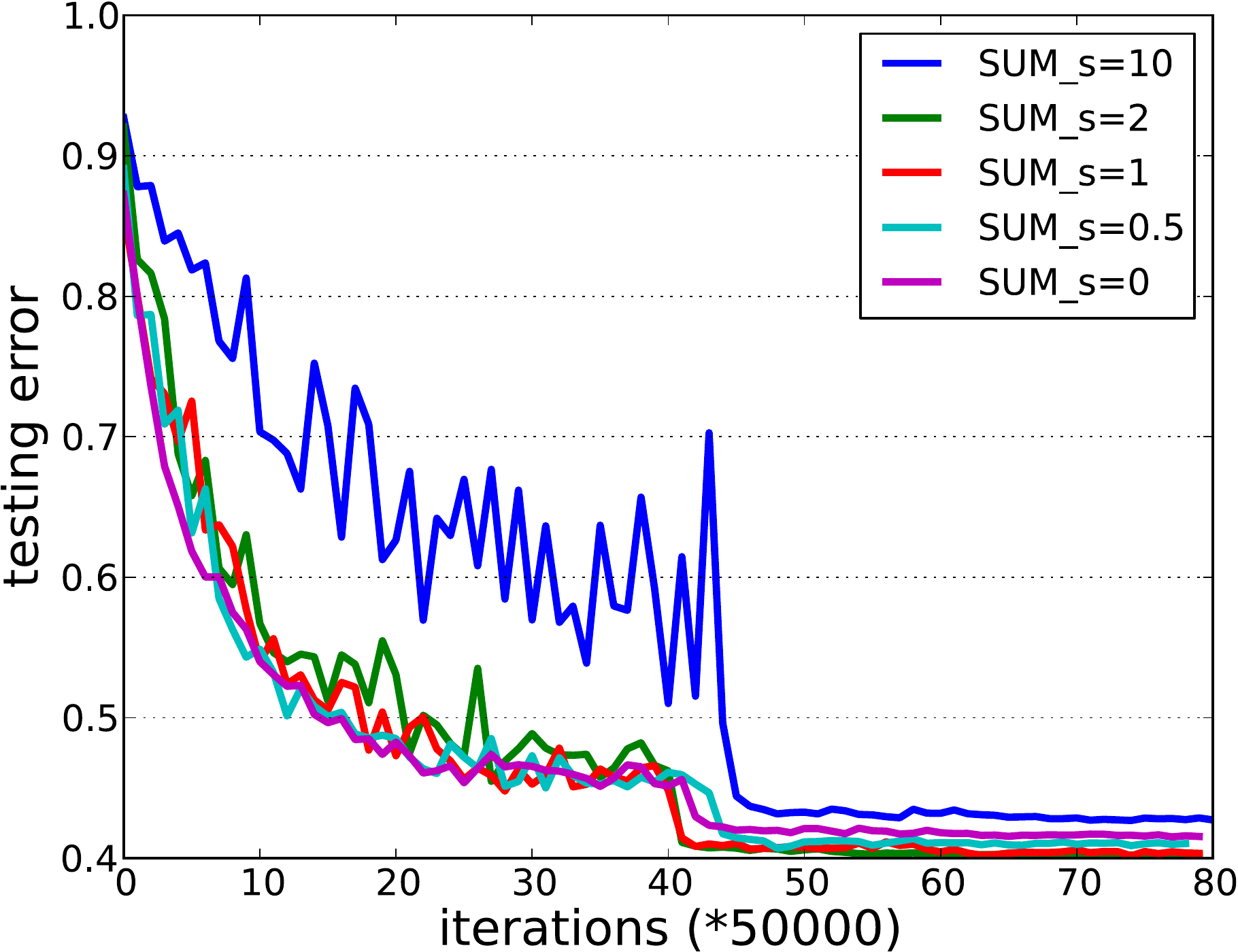}
  }
  \subfigure[abs(training error - testing error)]{
    \label{subfigure:abs_diff_different_s}
    \includegraphics[scale=0.29]{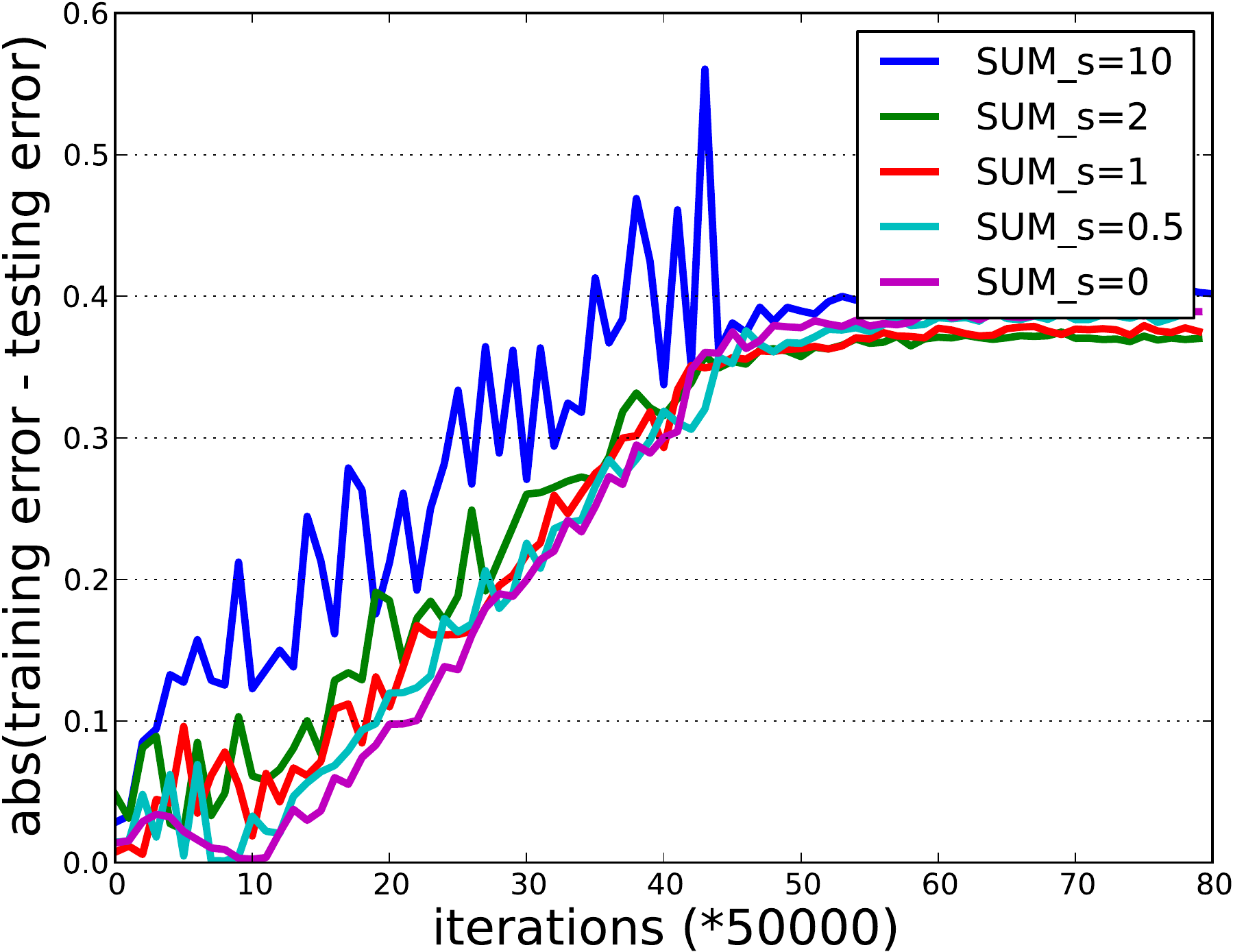}
  }
\caption{Training and testing error and their absolute difference on CIFAR-100 of SUM with different  $s$.}
\label{fig:different_s}
\end{figure*}

\section{Empirical Studies}\label{sec:exp}
In this section, we present empirical results on the non-convex optimization of deep neural networks.
We train a deep convolutional neural network (CNN) for classification on two benchmark datasets, i.e., CIFAR-10 and CIFAR-100.
Both datasets contain $50,000$ training images of size $32*32$ from $10$ classes (CIFAR-10) or $100$ classes (CIFAR-100) and $10,000$ testing images of the same size.
The employed CNN consists of 3 convolutional layers and 2 fully-connected layers.
Each convolutional layer is followed by a max pooling layer.
The output of the last fully-connected layer is fed into a $10$-class or $100$-class softmax loss function.
We emphasize that we do not intend to obtain the state-of-the-art prediction performance by trying different network structures and different engineering tricks, but instead focus our attention on verifying the theoretical analysis.
We compare the three variants of SUM, i.e., SHB, SNAG, and SG, which corresponds to $s=0$, $s=1$ and $s = 1/(1-\beta)$ in~(\ref{eqn:um}).
We fix the momentum constant $\beta=0.9$ and the regularization parameter of weights to $0.0005$.
We use a mini-batch of size $128$ to compute a stochastic gradient at each iteration.
All three methods use the same initialization.
We follow the procedure in~\cite{krizhevsky2012imagenet} to set the step size $\alpha$, i.e., initially giving a relatively large step size and and decreasing the step size by $10$ times after certain number of iterations when observing the performance on testing data saturates.

\textbf{Results on CIFAR-10.} We first present the convergence results of different methods with the best initial step size. In particular, for the initial step size, we search in a range ($\{0.001, 0.002, 0.005, 0.01, 0.02, 0.05\}$) for different methods and select the best one that yields the fastest convergence in training error.  In particular,  for SHB the best initial step size is $0.001$ and that for SNAG and SG is $0.01$.
In fact, a larger initial step size (e.g, $0.002$) for SHB gives a divergent result.
The training and testing error of different methods versus the number of iterations is plotted in Figure~\ref{fig:best_init_step_size}. 
This result is consistent with our convergence result in Theorem~\ref{thm:4}. 

Next, we plot the performance of different methods with the same initial step size $0.001$ in Figure~\ref{fig:same_init_step_size}. 
We report the training error, the testing error and their absolute difference in Figure~\ref{subfigure:train_error_cifar10},~\ref{subfigure:test_error_cifar10} and~\ref{subfigure:abs_diff_cifar10}, respectively. 
We use the absolute difference between the training and testing error as an estimate of the generalization error. 
We can see that the convergence of training error of the three methods are very close, which is consistent with our theoretical result in Theorem~\ref{thm:3}. 
Moreover, the behavior of the absolute difference between the training and testing error is also consistent with the theoretical result in Proposition~\ref{prop:ge}, i.e, SG has a larger generalization error than SNAG and SHB. 


\textbf{Results on CIFAR-100.}
We plot the training and testing error and their absolute difference of the three methods with the same initial step sizes ($0.002$) in Figure~\ref{subfigure:train_error_cifar100},~\ref{subfigure:test_error_cifar100} and~\ref{subfigure:abs_diff_cifar100}, respectively.
We observe similar trends in the training error and the generalization error, i.e., the three methods have similar performance (convergence speed) in training error but exhibit different degree of generalization error. The testing error curve shows that SNAG achieves the best prediction performance on the testing data. 

Finally, we present a comparison of SUM with different values of $s$ including variants besides SHB, SNAG  and SG.
In particular, we compare SUM with $s\in\{0, 0.5, 1, 2, 10\}$ and the same initial step size $0.002$.
Note that $s=0$ corresponds to SHB, $s=1$ corresponds to SNAG, $s=10$ corresponds to SG since $\beta=0.9$ and $s=2$ corresponds to a new variant.
The results on the CIFAR-100 data are shown in Figure~\ref{fig:different_s}.
From the results, we can observe that  the convergence of  training error for  different variants perform similarly.
For the generalization error, we observe a clear trend from $s=10$ to $s=0$ in that the generalization error decreases.

\section{Conclusion}
We have developed a unified framework of stochastic momentum methods that subsumes SHB, SNAG and SG as special cases, which have been widely adopted in training deep neural network. We also analyzed convergence of the training  for non-convex optimization and generalization error for learning of the unified stochastic momentum methods.
The unified framework and analysis bring more insights about differences between different methods and help explain experimental results for optimization deep neural networks.
In particular, the momentum term helps improve the generalization performance but not helps speed up the training process.   


\section*{Acknowledgements}

This work is partially supported by NSF-1545995, the Data to Decisions Cooperative Research Centre and ARC DP180100106.
Most work of Y. Yan was done when he was visiting the University of Iowa.

\bibliographystyle{named}
\bibliography{all}

\appendix

\section{Proof of Proposition~\ref{proposition:special_cases}}

We re-present Proposition~\ref{proposition:special_cases}.

\begin{reprop}
\label{reproposition:special_cases}
SUM~(\ref{eqn:um}) reduces to the three variants SG~(\ref{eqn:SG}), SHB~(\ref{eqn:shb1}) and SNAG~(\ref{eqn:sag1}) by setting $s = \frac{1}{1 - \beta}$, $s = 0$ and $s = 1$, respectively.
Particularly, the update of SG is
$
\x_{k+1} = \x_{k} - \frac{\alpha}{1 - \beta} \G(\x_{k}, \xi_{k}) ,
$
where the step size is $\frac{\alpha}{1 - \beta}$.
\end{reprop}

\begin{proof}

We prove the result by separately discussing three different values of $s$, i.e., $0$, $1$ and $\frac{1}{1-\beta}$.
We show that these three settings in fact correspond to SHB, SNAG and SG, respectively.

{\bf {1.}} When $s=0$, then $\y^s_{k+1} = \x_k$, the update of $\x_{k+1}$ becomes
\begin{align}\label{proof:proposition:special_cases:eqn:shb1}
\x_{k+1} = \x_k - \alpha \G(\x_k) + \beta(\x_{k} - \x_{k-1})  ,
\end{align}
which is exactly the update of the HB method in~(\ref{eqn:shb1}) and~(\ref{eqn:shb2}).

{\bf {2.}} When $s=1$, then $\y^s_{k+1} = \y_{k+1}$, then the update in~(\ref{eqn:um}) reduces to that in~(\ref{eqn:sag1}) or~(\ref{eqn:sag2}) of the NAG method.

{\bf {3.}} The last special case corresponds to using $s = \frac{1}{1-\beta}$ for SG.
We will show that the update of $\x_{k+1}$ is equivalent to
\begin{align}\label{eqn:gd}
\x_{k+1} = \x_k - \frac{\alpha}{1 - \beta}\G(\x_k), k\geq 0,
\end{align}
which is exactly the update of the gradient method but with a step size $\frac{\alpha}{1 - \beta}$.
First, we verify~(\ref{eqn:gd}) holds for $k=0$. 
From the updates in~(\ref{eqn:um}), we have
\begin{align*}
\x_1  & = \y_1 + \beta(\y^s_1 - \y^s_0) \\
&= \x_0 - \alpha \G(\x_0) + \beta(\x_0 - s\alpha \G(\x_0) - \x_0)\\
& =  \x_0 - \alpha \G(\x_0) - s \beta \alpha \G(\x_0) \\
&= \x_0 - \alpha \G(\x_0)(1 + \frac{\beta}{1-\beta}) = \x_0 - \frac{\alpha}{1-\beta}\G(\x_0)  .
\end{align*}
Then we show~(\ref{eqn:gd}) holds for any $k\geq 1$. 
From the updates in~(\ref{eqn:um}), we have
\begin{align*}
&\x_{k+1} - \x_k = - \alpha \G(\x_k) + \beta (\y^s_{k+1} - \y^s_k) \\
& = - \alpha \G(\x_k) + \beta(\x_k - s\alpha\G(\x_k) - \x_{k-1} + s\alpha \G(\x_{k-1}))  .
\end{align*}
Then
\begin{align*}
&\x_{k+1} - \x_k + s\alpha  \G(\x_k) = \beta(\x_k - \x_{k-1} + s\alpha \G(\x_{k-1}))\\
& + (s - 1 - \beta s)\alpha \G(\x_k)  .
\end{align*}
Since $s = \frac{1}{1-\beta}$, then $s - 1 - \beta s=0$, thus for any $k\geq 1$
\begin{align*}
\x_{k+1} - \x_k + s\alpha  \G(\x_k) = \beta(\x_k - \x_{k-1} + s\alpha \G(\x_{k-1})), 
\end{align*}
Therefore
\[
\x_{k+1} - \x_k + s\alpha \G(\x_k)= \beta^{k} (\x_1 - \x_0 + \frac{\alpha}{1-\beta}\G(\x_0) ) = 0 ,
\]
which leads to~(\ref{eqn:gd}).

\end{proof}

\section{Proof of Lemma~\ref{lem:k}}

We re-present Lemma~\ref{lem:k}.

\begin{relemma}\label{relem:k} 
Let $\p_k$ be
\begin{equation}\label{proof:eqn:p}
\p_k = \left\{
\begin{aligned}
&\frac{\beta}{1-\beta}(\x_k - \x_{k-1} + s\alpha \G(\x_{k-1})),\: k\geq 1\\
& 0,\quad k=0\\
\end{aligned}\right.  \tag{\ref{eqn:p}}
\end{equation}
and
\begin{equation}\label{proof:eqn:v}
\v_k = \frac{(1-\beta)}{\beta}\p_k   .  \tag{\ref{eqn:v}}
\end{equation}
Then for any $k\geq 0$, we have
\begin{align}
\label{proof:eqn:rec}
\x_{k+1} + \p_{k+1} & = \x_k + \p_k  - \frac{\alpha}{ 1- \beta}\G(\x_k)   ,  \tag{\ref{eqn:rec}} \\
\v_{k+1} & = \beta \v_k + ((1-\beta)s - 1) \alpha \G(\x_k)  \label{proof:eqn:rec2}   . \tag{\ref{eqn:rec2}}
\end{align}
\end{relemma}

\begin{proof}

Let us first write down the updates:
\begin{equation*}
\begin{aligned}
\y_{k+1} &  = \x_k - \alpha \G(\x_k)\\
\y^s_{k+1} & = \x_k - s\alpha \G(\x_k)\\
\x_{k+1} & = \y_{k+1} + \beta(\y^s_{k+1} - \y^s_k)
\end{aligned}
\end{equation*}
We can see that
\begin{align*}
\x_{k+1} &= \x_k - \alpha\G(\x_k)\nonumber \\
&+ \beta(\x_k - s\alpha\G(\x_k) - \x_{k-1} + s\alpha\G(\x_{k-1}))   . \nonumber
\end{align*}
If we define $\x_{-1} = \x_0$ and $\G(\x_{-1})=0$, the above equation holds for any $k\geq 0$. 
Similarly, we can write $\p_k$ as
\begin{align*}
\p_k = \frac{\beta}{1-\beta}(\x_k - \x_{k-1} + s\alpha\G(\x_{k-1}))
\end{align*}
for any $k\geq 0$.

Next, we prove that~(\ref{eqn:rec}) and~(\ref{eqn:rec2}) hold for any $k\geq 0$.
By the definition of $\p_k$, we have
\begin{align*}
&\x_{k+1} + \p_{k+1} = \x_{k+1} + \frac{\beta}{1-\beta}(\x_{k+1} - \x_k + s\alpha\G(\x_k))\\
&=\frac{1}{1-\beta}\x_{k+1} - \frac{\beta}{1-\beta}\left(\x_k  - s\alpha\G(\x_k)\right)\\
&\overset{\text{(\ref{eqn:um})}} {=}  \frac{1}{1-\beta}\left[\x_k - \alpha \G(\x_k)  \right.\\
&+\left.\beta(\x_k - s\alpha G(\x_k) - \x_{k-1}+ s\alpha G(\x_{k-1}))\right]  \\
&  -\frac{\beta}{1-\beta}\left(\x_k  - s\alpha\G(\x_k)\right)\\
& = \frac{1+\beta}{1-\beta}\x_k- \frac{1+s\beta}{1-\beta}\alpha \G(\x_k) \nonumber\\
& - \frac{\beta}{1-\beta}(\x_{k-1} - s\alpha \G(\x_{k-1}))\\
& - \frac{\beta}{1-\beta}(\x_k - s\alpha\G(\x_k))\\
& = \frac{1}{1-\beta} (\x_k - \alpha\G(\x_k)) -  \frac{\beta}{1-\beta}(\x_{k-1} - s\alpha \G(\x_{k-1})).
\end{align*}
Similarly, we have
\begin{align*}
\x_k + \p_k & = \frac{1}{1-\beta}\x_k - \frac{\beta}{1-\beta}(\x_{k-1} - s\alpha\G(\x_{k-1})).
\end{align*}
Thus
\begin{align*}
\x_{k+1} + \p_{k+1} = \x_k + \p_k - \frac{1}{1-\beta}\alpha\G(\x_k),
\end{align*}
which verifies~(\ref{eqn:rec}).

To verify~(\ref{eqn:rec2}), we use the definition of $\v_k$ and $\p_k$, and have
\begin{align*}
\v_{k+1} &= \frac{1-\beta}{\beta}\p_{k+1} = \x_{k+1} - \x_k + s\alpha\G(\x_k)\\
&\overset{\text{(\ref{eqn:um})}}{=} \x_k - \alpha\G(\x_k)\\
&+ \beta(\x_k - s\alpha\G(\x_k) - \x_{k-1} + s\alpha\G(\x_{k-1}))\\
&  - \x_k + s\alpha\G(\x_k)\\
& = \beta(\x_k - \x_{k-1}) + [s(1-\beta)-1]\alpha\G(\x_k)\\
& + \beta s\alpha\G(\x_{k-1})
\end{align*}
and
\begin{align*}
  & \beta\v_k +[(1-\beta)s - 1]\alpha\G(\x_k)  \nonumber\\
= & \beta(\x_k - \x_{k-1}) + \beta s\alpha \G(\x_{k-1})\\
& + [(1-\beta)s - 1]\alpha\G(\x_k)
\end{align*}
We can see that
\begin{align*}
\v_{k+1} = \beta\v_k + [(1-\beta)s  - 1]\alpha\G(\x_k),
\end{align*}
which verifies~(\ref{eqn:rec2}).

\end{proof}

\section{Proof of Lemma~\ref{lem:cm}}

We re-present Lemma~\ref{lem:cm}.

\begin{relemma}\label{relem:cm}
Given the update in~(\ref{eqn:um}), for any $k\geq 0$ we have
\begin{align}
    \x_{k+1}  = \x_{0} - \sum_{\tau=0}^{k} \left\{ \frac{1}{1-\beta} - \beta^{k-\tau+1}\frac{1-s(1-\beta)}{1-\beta} \right\}\alpha \grad(\x_{\tau})   .
\end{align}
\end{relemma}

\begin{proof}

Start with the update of~(\ref{eqn:um}).
We can rewrite the update by
\begin{align}
\x_{k+1} & = \x_{k} - \alpha \grad(\x_{k}) \nonumber\\
         & \quad + \beta ( \x_{k} - s\alpha\grad(\x_{k}) - (\x_{k-1} - s\alpha\grad(\x_{k-1}) ) )    \nonumber\\
    \Rightarrow \nonumber\\
\x_{k+1} & - (\x_{k} - s\alpha\grad(\x_{k})) \nonumber\\
         & = \beta (\x_{k} - (\x_{k-1} - s\alpha\grad(\x_{k-1}))) \nonumber\\
         & \quad + ((1 - \beta)s - 1)\alpha\grad(\x_{k})  \nonumber\\
    \Rightarrow \nonumber\\
\x_{k+1} & - (\x_{k} - s\alpha\grad(\x_{k})) \nonumber\\
         & = \sum_{\tau = 0}^{k} ((1 - \beta)s - 1)\beta^{k - \tau} \alpha \grad (\x_{\tau})  \nonumber
\end{align}

\begin{align}\label{eq:proof_lemma2_2}
\Rightarrow & \nonumber\\
\x_{k+1} & = \x_{k} - s\alpha\grad(\x_{k}) + \sum_{\tau = 0}^{k} ((1 - \beta)s - 1)\beta^{k - \tau} \alpha \grad (\x_{\tau}) \nonumber\\
\Rightarrow & \nonumber\\
\x_{k+1} & = \x_{0} - \sum_{\tau = 0}^{k} s \alpha \grad (\x_{\tau}) \nonumber\\
         & \quad + \sum_{\tau = 0}^{k} \sum_{i=0}^{\tau} ((1 - \beta)s - 1) \beta^{\tau - i}\alpha\grad(\x_{i})  \nonumber\\
\Rightarrow & \nonumber\\
\x_{k+1} & = \x_{0} - \sum_{\tau = 0}^{k} s \alpha \grad (\x_{\tau}) \nonumber\\
         & \quad + \frac{(1 - \beta)s - 1}{1 - \beta} \sum_{\tau = 0}^{k} (1 - \beta^{k - \tau + 1})\alpha\grad(\x_{\tau}) \nonumber\\
\Rightarrow & \nonumber\\
\x_{k+1} & = \x_{0}   \nonumber\\
         - & \sum_{\tau = 0}^{k} \left\{ \frac{1}{1 - \beta} - \beta^{k-\tau+1} \frac{1 - s(1 - \beta)}{1 - \beta} \right\} \alpha\grad(\x_{\tau})  .
\end{align}

\end{proof}

\section{Proof of Theorem~\ref{thm:3}}  
\label{section:proof:thm:3}

Before proving Theorem~\ref{thm:3}, we present two key lemmas.

\begin{lemma}\label{lem:1} 
Let $\z_k = \x_k + \p_k$.
For SUM,  we have for any $k\geq 0$,
\begin{align*}
&\E[f(\z_{k+1})- f(\z_k)] \leq\frac{1}{2L}\E [\| \nabla f(\z_k) - \nabla f(\x_k)\|^2]\\
&+ \left(\frac{L\alpha^2}{(1-\beta)^2} -  \frac{\alpha}{1-\beta}\right)\E[\|\nabla f(\x_k)\|^2] + \frac{L\alpha^2\sigma^2}{2(1-\beta)^2}  .
\end{align*}
\end{lemma}

\begin{proof}
Let $\delta_k  = \G_k - \nabla f(\x_k)$. Then $\E[\delta_k]=0$. Beginning by exploring  the smoothness of $f(\x)$ we have
\begin{align*}
&f(\z_{k+1}) \leq f(\z_k) + \nabla f(\z_k)^{\top}(\z_{k+1} - \z_k) \\
&+ \frac{L\|\z_{k+1} - \z_k\|^2}{2}\\
&\overset{\text{(\ref{eqn:rec})}}{=}  f(\z_k) - \frac{\alpha}{1-\beta}\nabla f(\z_k)^{\top}\G_k +  \frac{L}{2}\frac{\alpha^2}{(1-\beta)^2}\|\G_k\|^2
\end{align*}
\begin{align*}
&  = f(\z_k) - \frac{\alpha}{1-\beta}\nabla f(\z_k)^{\top}(\delta_k+\nabla f(\x_k)) \\
&+  \frac{L}{2}\frac{\alpha^2}{(1-\beta)^2}\|\G_k\|^2\\
&  = f(\z_k) - \frac{\alpha}{1-\beta}\nabla f(\z_k)^{\top}\delta_k - \frac{\alpha}{1-\beta} \nabla f(\z_k)^{\top}\nabla f(\x_k)\\
&+ \frac{L}{2}\frac{\alpha^2}{(1-\beta)^2}\|\G_k\|^2\\
&  = f(\z_k) - \frac{\alpha}{1-\beta}\nabla f(\z_k)^{\top}\delta_k \\
&- \frac{\alpha}{1-\beta} (\nabla f(\z_k) - \nabla f(\x_k))^{\top}\nabla f(\x_k) \\
&- \frac{\alpha}{1-\beta}\|\nabla f(\x_k)\|^2 + \frac{L}{2}\frac{\alpha^2}{(1-\beta)^2}\|\delta_k + \nabla f(\x_k)\|^2 .
\end{align*}

Taking expectation on both sides
\begin{align*}
&\E[f(\z_{k+1})- f(\z_k)]\\
&\leq \E\left[ -\frac{\alpha}{1-\beta} (\nabla f(\z_k) - \nabla f(\x_k))^{\top}\nabla f(\x_k) \right.\\
&\left.-  \frac{\alpha}{1-\beta}\|\nabla f(\x_k)\|^2  +\frac{L}{2}\frac{\alpha^2}{(1-\beta)^2}\|\nabla f(\x_k)\|^2 \right]\nonumber\\
&  +\frac{L}{2}\frac{\alpha^2}{(1-\beta)^2}\E[\|\delta_k\|^2]\nonumber\\
&= \E\left[ -\frac{\alpha}{1-\beta} (\nabla f(\z_k) - \nabla f(\x_k))^{\top}\nabla f(\x_k) \right] + \\
&\left(\frac{L}{2}\frac{\alpha^2}{(1-\beta)^2} -  \frac{\alpha}{1-\beta}\right)\E[\|\nabla f(\x_k)\|^2] + \frac{L\alpha^2}{2(1-\beta)^2}\sigma^2\\
&\leq\E \left[\frac{1}{2L}\| \nabla f(\z_k) - \nabla f(\x_k)\|^2 + \frac{L\alpha^2}{2(1-\beta)^2}\|\nabla f(\x_k)\|^2\right]\nonumber\\
&+ \left(\frac{L}{2}\frac{\alpha^2}{(1-\beta)^2} -  \frac{\alpha}{1-\beta}\right)\E[\|\nabla f(\x_k)\|^2] + \frac{L\alpha^2}{2(1-\beta)^2}\sigma^2       , \nonumber   
\end{align*}
where the last inequality uses the Cauchy-Schwarz inequality.
\end{proof}

\begin{lemma}\label{lem:2}
For SUM, we have for any  $k\geq 0$,
\begin{align*}
\E[&\|\nabla f(\z_k) - \nabla f(\x_k)\|^2]  \nonumber\\
\qquad & \leq \frac{L^2\beta^2((1-\beta)s -1)^2\alpha^2(G^2+\sigma^2)}{(1-\beta)^4}.
\end{align*}
\end{lemma}

\begin{proof}
\begin{align*}
&\|\nabla f(\z_k) - \nabla f(\x_k)\|^2\leq L^2\|\z_k - \x_k\|^2 = L^2 \|\p_k\|^2     \\
& \overset{\text{(\ref{eqn:v})}}{=}   \frac{L^2\beta^2}{(1-\beta)^2}\E[\|\v_k\|^2]   .\\  
\end{align*}
Recall the recursion in~(\ref{eqn:rec2}):
\[
\v_{k+1} = \beta\v_k +  ((1-\beta)s-1)\alpha \G_k  .
\]
Note that $\v_0=0$.  Denote by $\hat\alpha = \alpha ((1-\beta)s-1)$.
By induction,  we can show that
\[
\v_k = \hat\alpha\sum_{i=0}^{k-1}\beta^{k-1-i}\G_i=  \hat\alpha \sum_{i=0}^{k-1}\beta^{i}\G_{k-1-i}
\]
Let $\Gamma_{k-1}= \sum_{i=0}^{k-1}\beta^i = \frac{1-\beta^k}{1-\beta}$. Then
\begin{align*}
\|\v_k\|^2&=\left\|\sum_{i=0}^{k-1}\frac{\beta^i}{\Gamma_{k-1}}\hat\alpha \G_{k-1-i}\right\|^2\Gamma_{k-1}^2\\
&\leq\Gamma_{k-1}^2 \sum_{i=0}^{k-1}\frac{\beta^i}{\Gamma_{k-1}}\hat\alpha^2\|\G_{k-1-i}\|^2\\
&=\Gamma_{k-1}\sum_{i=0}^{k-1}\beta^i\hat\alpha^2\|\G_{k-1-i}\|^2  .
\end{align*}
Then
\begin{align*}
&\E[\|\v_k\|^2]\leq \Gamma_{k-1}\sum_{i=0}^{k-1}\beta^i\hat\alpha^2(G^2 + \sigma^2)\\
&= \Gamma_{k-1}^2\hat\alpha^2(G^2+\sigma^2)\leq\frac{\alpha^2((1-\beta)s-1)^2(G^2+\sigma^2)}{(1-\beta)^2}  .
\end{align*}
Then
\begin{align*}
\|\nabla f(\z_k)& - \nabla f(\x_k)\|^2
 \leq \frac{L^2\beta^2}{(1-\beta)^2}\E[\|\v_k\|^2]\\
& \leq \frac{L^2\beta^2((1-\beta)s-1)^2\alpha^2(G^2+\sigma^2)}{(1-\beta)^4}  .
\end{align*}
\end{proof}

Here we re-present Theorem~\ref{thm:3}.

\begin{rethm}~\label{proof:thm:3}
(Convergence of SUM)
Suppose $f(\x)$ is a non-convex and $L$-smooth function, $\E[\|\G(\x; \xi) - \nabla f(\x)\|^2]\leq \sigma^2$ and $\|\nabla f(\x)\|\leq G$ for any $\x$.
 Let update~(\ref{eqn:um}) run for $t$ iterations with $\G(\x_k; \xi_k)$. By setting $\alpha =\min\{\frac{1-\beta}{2L}, \frac{C}{\sqrt{t+1}}\}$ we have
\begin{align*}
&\min_{k=0,\ldots, t}\E[\|\nabla f(\x_k)\|^2]\\
&\leq \frac{2(f(\x_0) - f_*)(1-\beta)}{t+1}\max\left\{\frac{2L}{1-\beta}, \frac{\sqrt{t+1}}{C}\right\} \\
&+ \frac{C}{\sqrt{t+1}}\frac{L\beta^2((1-\beta)s -1)^2(G^2+\sigma^2) + L\sigma^2(1-\beta)^2}{(1-\beta)^3} .
\end{align*}
\end{rethm}

\begin{proof}\label{proof:theorem1}
Let $B, B'$ be defined as
\begin{align*}
B &=  \frac{\alpha}{(1-\beta)} - \frac{L\alpha^2}{(1-\beta)^2}>0\\
B'&=\frac{L\beta^{2}((1-\beta)s-1)^2\alpha^2(G^2+\sigma^2)}{2(1-\beta)^4} + \frac{L\alpha^2\sigma^2}{2(1-\beta)^2}  .
\end{align*}
Lemma~\ref{lem:1} and Lemma~\ref{lem:2} imply that
\begin{align*}
\E[f(\z_{k+1}) - f(\z_k)]\leq -B\E[\|\nabla f(\x_k)\|^2]  +B'  .
\end{align*}
By summing the above inequalities for $k=0,\ldots, t$ and noting that $\alpha< \frac{1-\beta}{L}$,
\begin{align*}
&B\sum_{k=0}^t\E[\|\nabla f(\x_k)\|^2]\leq \E[f(\z_0) - f(\z_{t+1})] + (t+1)B'\\
&\leq \E[f(\z_0) - f_*] + (t+1)B'   .
\end{align*}
Then
\begin{align*}
\min_{k=0,\ldots, t}\E[\|\nabla f(\x_k)\|^2]\leq \frac{f(\z_0) - f_*}{(t+1)B} + \frac{B'}{B}.
\end{align*}
Assume $\alpha \leq \frac{1-\beta}{2L}$, then  $B =\frac{\alpha}{1-\beta} - \frac{\alpha^2L }{(1-\beta)^2} \geq \frac{\alpha}{2(1-\beta)}$.
Then
\begin{align}
\min_{k=0 ,\ldots, t} & \E[\|\nabla f(\x_k)\|^2] \nonumber\\
          & \leq \frac{2(f(\z_0) - f_*)(1-\beta)}{\alpha (t+1)} + \frac{2(1-\beta)}{\alpha}B' .
\end{align}
Noting that $\alpha =\min\{\frac{1-\beta}{2L}, \frac{C}{\sqrt{t+1}}\}$, we can have
\begin{align*}
&\min_{k=0,\ldots, t}\E[\|\nabla f(\x_k)\|^2]\\
&\leq \frac{2(f(\z_0) - f_*)(1-\beta)}{t+1}\max\left\{\frac{2L}{1-\beta}, \frac{\sqrt{t+1}}{C}\right\}\\
& + \frac{C}{\sqrt{t+1}}\frac{L\beta^{2}((1-\beta)s-1)^2(G^2+\sigma^2) + L(1-\beta)^2\sigma^2}{(1-\beta)^3} .
\end{align*}
We then complete the proof by noting that $\z_0 = \x_0$.
\end{proof}

\section{Proof of Theorem 2}

As in Section~\ref{section:proof:thm:3}, with a slightly different analysis from that of Lemma~\ref{lem:1}, we can have the following lemma.

{\bf{Lemma 5}}
\label{lem:3}
\emph{
Let $\z_k = \x_k + \p_k$. For SUM,  we have for any $k\geq 0$,
}
\begin{align*}
&\E[f(\z_{k+1})- f(\z_k)] \leq \frac{L\alpha^2\sigma^2}{2(1-\beta)^2} \\
&+ \frac{1}{2L((1-\beta)s-1)^2}\E [\| \nabla f(\z_k) - \nabla f(\x_k)\|^2]\\
&+ \left(\frac{[1+((1-\beta)s-1)^2]\alpha^2L}{2(1-\beta)^2} -  \frac{\alpha}{1-\beta}\right)\E[\|\nabla f(\x_k)\|^2] .
\end{align*}

\begin{proof}
We can follow the same analysis as in the proof of Lemma~\ref{lem:1} and get \begin{align*}
&\E[f(\z_{k+1})- f(\z_k)]\\
&\leq  \E\left[ -\frac{\alpha}{1-\beta} (\nabla f(\z_k) - \nabla f(\x_k))^{\top}\nabla f(\x_k) \right] \\
&+ \left(\frac{L}{2}\frac{\alpha^2}{(1-\beta)^2} -  \frac{\alpha}{1-\beta}\right)\E[\|\nabla f(\x_k)\|^2] + \frac{L\alpha^2\sigma^2}{2(1-\beta)^2}\\
&\leq\frac{1}{2}\E \left[\frac{1}{L((1-\beta)s-1)^2}\| \nabla f(\z_k) - \nabla f(\x_k)\|^2 \right.\\
&\left.+ \frac{L\alpha^2((1-\beta)s-1)^2}{(1-\beta)^2}\|\nabla f(\x_k)\|^2\right]\\
&+ \left(\frac{L}{2}\frac{\alpha^2}{(1-\beta)^2} -  \frac{\alpha}{1-\beta}\right)\E[\|\nabla f(\x_k)\|^2] + \frac{L\alpha^2\sigma^2}{2(1-\beta)^2}\\
&=\frac{1}{2L((1-\beta)s-1)^2}\E[\| \nabla f(\z_k) - \nabla f(\x_k)\|^2]\\
&+ \left(\frac{\alpha^2L[1+((1-\beta)s-1)^2]}{2(1-\beta)^2} -  \frac{\alpha}{1-\beta}\right)\E[\|\nabla f(\x_k)\|^2] \\
&+ \frac{L\alpha^2\sigma^2}{2(1-\beta)^2}   .
\end{align*}
\end{proof}

With Lemma~5 and Lemma~\ref{lem:2} and a similar analysis as that for Theorem~\ref{thm:3}, we can easily prove Theorem~\ref{thm:4}.

\section{Proof of Proposition~\ref{prop:ge}}

We re-present Proposition~\ref{prop:ge}.

\begin{prop3*}\label{reprop:ge}
Assume that $\|\nabla\ell(\x, \data)\|_2\leq G$ for any $\x$ and $\data$ and $\ell(\x,\data)$ is $L$-smooth w.r.t $\x$. For two data sets $\S,\S'$ that differs at one example, let $\x_t$ and $\x'_t$ denote the $t$-th iterates of running SUM for the empirical risk minimization on $S$ and $S'$, we have
\[
\Delta_{t+1} \leq \sum_{k=0}^{t}\frac{2\alpha G}{n}\eta^t_k + \left(1-\frac{1}{n}\right)  \sum_{k=0}^{t}
\alpha L\eta_k^t \Delta_k
\]
with $\Delta_0=0$, where $
\eta_k^t = \frac{1}{1-\beta} - \beta^{t-k+1}\frac{1-s(1-\beta)}{1-\beta}$.
\end{prop3*}

\begin{proof}\label{proof:proposition2}
Recall that in Lemma~\ref{lem:cm}, we have
\begin{align*}
\x_{t+1} & = \x_{0} - \sum_{k=0}^{t} \eta^{t}_{k}\alpha \G(\x_{k}) \\
&=\x_{0} - \sum_{k=0}^{t} \eta^{t}_{k}\alpha \nabla \ell(\x_{k}, \data_{i_k})  ,
\end{align*}
where $\eta^{t}_{k} = \frac{1}{1-\beta} - \beta^{t-k+1}\frac{1-s(1-\beta)}{1-\beta}.$

Then we could upper bound of $\Delta_{t+1}$ as follows
\begin{align}\label{eq:proof_proposition2_1}
  \Delta_{t+1}
               & =    \E \| \x_{t+1} - \x_{t+1}' \|  \nonumber\\
               & =    \E \| \x_{0} - \sum_{k = 0}^{t} \eta^{t}_{k} \alpha \nabla \ell(\x_{k}, \data_{i_{k}}) - (\x_{0}' - \sum_{k = 0}^{t} \eta^{t}_{k}\alpha \nabla\ell(\x_{k}', \data_{i_{k}'}) ) \|  \nonumber\\
               & \leq \E \| \x_{0} - \x_{0}' ||_{2} + \sum_{k = 0}^{t}\eta^{t}_{k} \alpha \E \| \nabla\ell(\x_{k}, \data_{i_{k}}) - \nabla\ell(\x_{k}', \data_{i_{k}'})||   \nonumber\\
               & =    \sum_{k = 0}^{t}\eta^{t}_{k} \alpha \E || \nabla\ell(\x_{k}, \data_{i_{k}}) - \nabla\ell(\x_{k}', \data_{i_{k}'})\| \nonumber\\
\end{align}
To bound the expectation term on the R.H.S, we can use the fact that with probability $1/n$, $i_k\neq i_{k}'$ and $\| \nabla\ell(\x_{k}, \data_{i_{k}}) - \nabla\ell(\x_{k}', \data_{i_{k}'})\|\leq 2G$ due to Lipschitz continuity of $\ell(\x,\data)$, and with probability $1-1/n$, $i_k=i_{k}'$ and $\| \nabla\ell(\x_{k}, \data_{i_{k}}) - \nabla\ell(\x_{k}', \data_{i_{k}'})\|\leq L\|\x_k-\x_k'\|$ due to the smoothness of $\ell(\x,\z)$. Therefore,
\begin{align}
               \Delta_{t+1}& \leq     \sum_{k = 0}^{t}\eta^{t}_{k} \alpha \left\{ (1-\frac{1}{n})L \E || \x_{k} - \x_{k}' ||_{2} + \frac{2G}{n} \right\} \nonumber\\
               & =    \sum_{k = 0}^{t}\eta^{t}_{k} \alpha \left\{ (1-\frac{1}{n})L \Delta_{k} + \frac{2G}{n} \right\} \nonumber\\
               & =    \sum_{k = 0}^{t} \frac{2\alpha G}{n} \eta^{t}_{k} + \Big( 1 - \frac{1}{n} \Big)  \sum_{k=0}^{t}\alpha L \eta^{t}_{k}  \Delta_{k}.
\end{align}
\end{proof}

\end{document}